\documentclass{IEEEojcsys}

\usepackage[colorlinks,urlcolor=blue,linkcolor=blue,citecolor=blue]{hyperref}

\usepackage{color,array}

\usepackage{graphicx}

\usepackage{cite}
\usepackage{amsmath,amssymb,amsfonts}
\usepackage{amsthm}
\usepackage{mathtools}
\usepackage{algorithmic}
\usepackage{graphicx}
\usepackage{textcomp}
\usepackage{xcolor}
\usepackage[export]{adjustbox}
\usepackage[utf8]{inputenc}
\usepackage[english]{babel}
\usepackage{graphicx}
\usepackage[linesnumbered,ruled,vlined]{algorithm2e}
\usepackage{float}
\usepackage{subcaption}
\usepackage[english]{babel}
\usepackage{blindtext}
\usepackage{bm}
\usepackage{tabularx, multirow}
\usepackage{comment}
\usepackage{subcaption}
\SetKwInput{KwInput}{Input}                
\SetKwInput{KwOutput}{Output}              

\DeclareCaptionLabelSeparator{none}{ }
\SetCommentSty{mycommfont}

\DeclareMathOperator{\sgn}{sgn}

\def\BibTeX{{\rm B\kern-.05em{\sc i\kern-.025em b}\kern-.08em
    T\kern-.1667em\lower.7ex\hbox{E}\kern-.125emX}}
\jvol{00}
\jnum{XX}
\paper{1234567}
\pubyear{XXXX}
\receiveddate{XX XX XX}
\accepteddate{XX XX XX}
\publisheddate{XX XX XX}
\currentdate{19 December 2023}
\doiinfo{OJCSYS.2021.Doi Number}

\newtheorem{theorem}{Theorem}

\newtheorem{remark}{Remark}
\begin{document}

\sptitle{Article Category}

\title{Self-navigation in crowds: An invariant set-based approach} 


\author{VEEJAY KARTHIK J\affilmark{1}}

\author{$\;$LEENA VACHHANI\affilmark{1}} 


\affil{Indian Institute of Technology Bombay, Mumbai 400076, India }

\corresp{CORRESPONDING AUTHOR: Leena Vachhani (e-mail: \href{mailto:leena.vachhani@iitb.ac.in}{leena.vachhani@iitb.ac.in})}
\authornote{This work was supported by the Prime Minister's Research Fellowship (PMRF).\\ PMRF ID: 1302088}

\markboth{SELF-NAVIGATION IN CROWDS: AN INVARIANT SET-BASED APPROACH}{VEEJAY. K. J {\itshape ET AL}.}

\begin{abstract}
Self-navigation in non-coordinating crowded environments is formidably challenging within multi-agent systems consisting of non-holonomic robots operating through local sensing. Our primary objective is the development of a novel, rapid, sensor-driven, self-navigation controller that directly computes control commands to enable safe maneuvering while coexisting with other agents. We propose an input-constrained feedback controller meticulously crafted for non-holonomic mobile robots and the characterization of associated invariant sets. The invariant sets are the key to maintaining stability and safety amidst the non-cooperating agents. We then propose a planning strategy that strategically guides the generation of invariant sets toward the agent's intended target. This enables the agents to directly compute theoretically safe control inputs without explicitly requiring pre-planned paths/trajectories to reliably navigate through crowded multi-agent environments. The practicality of our technique is demonstrated through hardware experiments, and the ability to parallelize computations to shorten computational durations for synthesizing safe control commands. The proposed approach finds potential applications in crowded multi-agent scenarios that require rapid control computations based on perceived safety bounds during run-time.
\end{abstract}


\maketitle

\section{INTRODUCTION}

Navigation in crowded environments is a formidably challenging task due to the inherent uncertainties and unpredictable nature of moving agents in the robot's vicinity. Since predicting the trajectories of other moving agents in the vicinity is often very difficult during run-time in such scenarios, finding safe paths/trajectories for navigation (in standard decoupled plan-control approaches) becomes computationally super complex.  As exact models of evolving scenarios are often unavailable, the use of onboard sensors, such as LiDAR and cameras, for real-time perception is vital for ensuring collision avoidance and safety. However, these sensor measurements have inherent uncertainties, and in dynamic scenarios, the reliability of sensed information is limited to short durations. Noisy sensor measurements and occlusions of agents (blocked views) also often lead to inaccurate interpretations of the environment, and they pose challenges for ensuring safety during navigation. This necessitates robust and safe approaches to planning and control for navigation. Consequently, understanding the intricate interplay between these two domains is crucial as each contributes significantly to safe robot navigation in crowded environments.

In a nutshell, determining ``what" is the required motion to accomplish the given task specification is the planning problem, and ``how" to execute this required motion with the actuators of the system is the controls problem. Self-navigation is the ability of a robot to autonomously sense, plan, and control its movements in an environment without external guidance. This capability allows the robot to navigate, avoid obstacles, and reach specific destinations independently.

In the context stated earlier, rapid motion-planning algorithms capable of generating provably safe plans under uncertainties in the onboard sensory information for quick execution are inevitable for operation in dynamic, crowded settings. Furthermore, compliance of the generated motion plan with the system dynamics is essential for its faithful execution through the actuators of the system. Since predominant mobile robotic systems such as legged and wheeled robots are underactuated, their planning and control problems are non-trivial. In crowded navigation scenarios, collision avoidance \cite{CooperativeCollisionAvoidanceNonHolRobots} is paramount to ensure safety. The state-of-the-art methods available in the literature for tackling such problems are invariably through online optimization/ optimal control strategies. In recent times, learning control policies through deep reinforcement learning methods have also emerged as tools for tackling crowd navigation problems.

Trajectory optimization (TO) and its variants are prominent techniques for generating motion plans amiable to the system dynamics for robotic systems available in the literature. The TO-based methods \cite{TrajOptMultiRobotNav} deployed in multi-agent navigation applications operate through motion prediction where every agent tries to predict a continuous trajectory for its neighbors (through the communication of states among agents) to plan its own for ensuring collision avoidance. Model predictive control (MPC) schemes \cite{ExperimentalSafeMPC,MultiTrajectoryMPCUAV,NMPCMultiRobot,TubeMPCMultiRobot} are class a of optimization-based feedback controllers with capabilities to encode the system dynamics and their associated constraints (on both states and inputs) in their online, finite horizon optimization problem for generating amiable control inputs. Further, the planning horizons are not fixed but are practically limited by the available onboard computational resources to safely maneuver through the dynamically varying collision-free regions ahead of the robot.

Set invariance is the fundamental property that is sought in control theory to ensure safety in systems. In recent times, set invariance is enforced through control barrier functions (CBFs) \cite{CBFTheoryApplications,CBFInputtoStateSafety,BarrierFunctionsMultiRobotNavigation} and barrier certificates \cite{BarrierCertificatesMultiRobot} through optimization-based formulations to ensure safe navigation by minimally modifying their nominal stabilizing controllers. The key challenges involved in ensuring safety through such formulations are that the construction of CBFs is non-trivial when there are constraints on inputs \cite{ICCBF}, when the constructed CBF has a relative degree more than one, and when there are bounded disturbances\cite{BREEDEN2023111109}. Schemes combining CBFs with MPC and motion prediction modules have been proposed for tackling robot navigation in crowds \cite{NMPCCBFCrowd}. However, scalability and computational duration are key factors to be considered in such navigation paradigms, as they degrade in performance with an increasing number of agents in the crowd.

MPC-based approaches require significant computational capabilities as the solution to the online optimization problem often has non-convex constraints (distance constraints) arising due to obstacles in the workspaces, and non-linearities in the system model. Determining the global optimizer becomes non-trivial, and only locally optimal solutions are possible. Real-time demonstrations of such approaches are carried out through approximations to ease the computational complexities \cite{NMPCApproxTraffic} during run-time. Furthermore, striking a balance in choosing tuning parameters such as prediction horizon becomes crucial since, short horizons might result in unaccountability towards long-term effects, and long horizons increase the computational complexity. Consequently, their reliability is subject to prediction accuracy (through numerical integration of discretized system models), availability of onboard computational resources, and careful selection of the tuning parameters. Also, in MPC formulations where the CBF conditions are set as constraints \cite{CBFMPC}, it is critical to ensure that the online optimization problem does not become infeasible during run-time to prevent adverse consequences. 

In crowded environments consisting of humans, social navigation \cite{SociallyAwareNavigation} approaches try to predict human behaviors for incorporation into the planning process to enable natural and safe interactions. The key challenge is to accurately model human behavior (which often varies with the diversity of the crowd) to enable safe operation. Consequently, sophisticated data-driven motion prediction models have been developed to enhance predictions for safety. However, recent works \cite{SOTApredictionNavperformance} suggest that the existing state-of-the-art motion prediction methods do not provide significant enhancements in navigation performance yet.  

Deep reinforcement learning-based methods \cite{CrowdAwareDRL} utilize neural networks to directly learn control policies from experience, and are often suited for diverse crowd conditions and complex environments. They require vast training data, and their learned control policies often tend to be biased and potentially unsafe. Further, since neural networks are involved, it also becomes infeasible to reason about the source of potential unsafe behaviors that might be induced. As real-time deployment capabilities and safety are non-negotiables for operation in practical crowded scenarios, research into alternate sensor-based planning paradigms becomes a necessity.

\textbf{To circumvent the computational complexity in computing amiable reference trajectories/ feedback control inputs through solving online optimizations by probing through the state space of the system at each instance, we seek to answer the question: If we know beforehand how the system moves for a given reference, can we plan accordingly? More specifically, if concrete geometric characterizations of the regions on the state-space (state constraints) where naturally induced system trajectories are going to evolve can be pre-determined through a pre-synthesized controller, can we plan how to compute control inputs through such controllers so that the naturally induced system trajectories respect the bounds of safety identified during run-time? The proposed work answers these questions and provides the following key benefits.}
\begin{itemize}
    \item As the motion constraints associated with the system dynamics, potential disturbances during operation, and input constraints due to actuator limitations are inherently taken into consideration during the offline controller design phase, their inherent complexities need not be dealt with during the online planning phase.     
    \item When the pre-synthesized controller exhibits closed-form solutions, control inputs can be readily computed through quick evaluation of algebraic expressions during run-time.
    \item When specific geometric characterizations of regions (state constraints) where the induced system trajectories are going to evolve are available beforehand, tools from computational geometry can be deployed to accelerate online computations on the hardware level during run-time. Such geometric characterizations also allow us to deal with measurement uncertainties deterministically.
\end{itemize}
Since feedback controllers are the fundamental building blocks in this work, the careful design of such controllers is a key aspect that needs to be meticulously dealt with. Essential aspects such as input constraints (based on the physical actuator limitations), and associated set invariance need to be pre-determined during the control design phase for seamless operation during run-time.

Feedback control synthesis through Lyapunov function-based techniques \cite{khalil2015nonlinear} has been the cornerstone of non-linear control design for decades. Feedback controllers synthesized through such techniques exhibit closed-form expressions, and the induced trajectories via such feedback controllers are positively invariant over a set \cite{SetInvarianceControl} whose geometric structure could be explicitly characterized through the underlying lyapunov function. Such invariant sets can be thought of as the geometric manifestation of the feedback controller itself. The stability and robustness of the feedback controllers that are designed through such techniques are proved through the negative-definite nature of the time derivative of the underlying Lyapunov function. 

Further, the generic integrated planning and control approaches in the literature seek to exploit these geometric structures to obtain motion plans directly in the space of controls for applications such as the navigation of wheeled mobile agents \cite{Conner2011}. Given the exact map of a static operating environment, sequential composition \cite{SeqComp,MultiRobotSeqComp,NonHolCartNav} is an offline strategy through which invariant sets are composed (sequenced) together such that the induced system trajectories (naturally abide by the motion constraints) through the sequential triggering of the associated feedback controller, safely converge to the target point of navigation. 

In this context, invariant sets (artifacts from lyapunov function-based non-linear control designs) are utilized to directly compute safe control inputs for navigation. These sets are well-defined geometric characterizations of the region where the naturally induced system trajectories are expected to evolve when deploying the feedback controller. As the bounds of local navigable regions can be readily obtained through onboard sensing, the problem of rapidly computing control inputs for safe navigation is tackled by strategically identifying and guiding the invariant sets. The proposed invariant set-based navigator avoids the need for explicit paths/trajectories for navigation in crowded multi-agent environments, and we formulate the problem in the next section. To the best of the author's knowledge, this is the first work in the literature that accomplishes guaranteed safe navigation in crowded multi-agent scenarios via invariant sets through onboard sensing.

\subsection{Problem Formulation}
\begin{figure}
    \centering
    \includegraphics[width = \columnwidth]{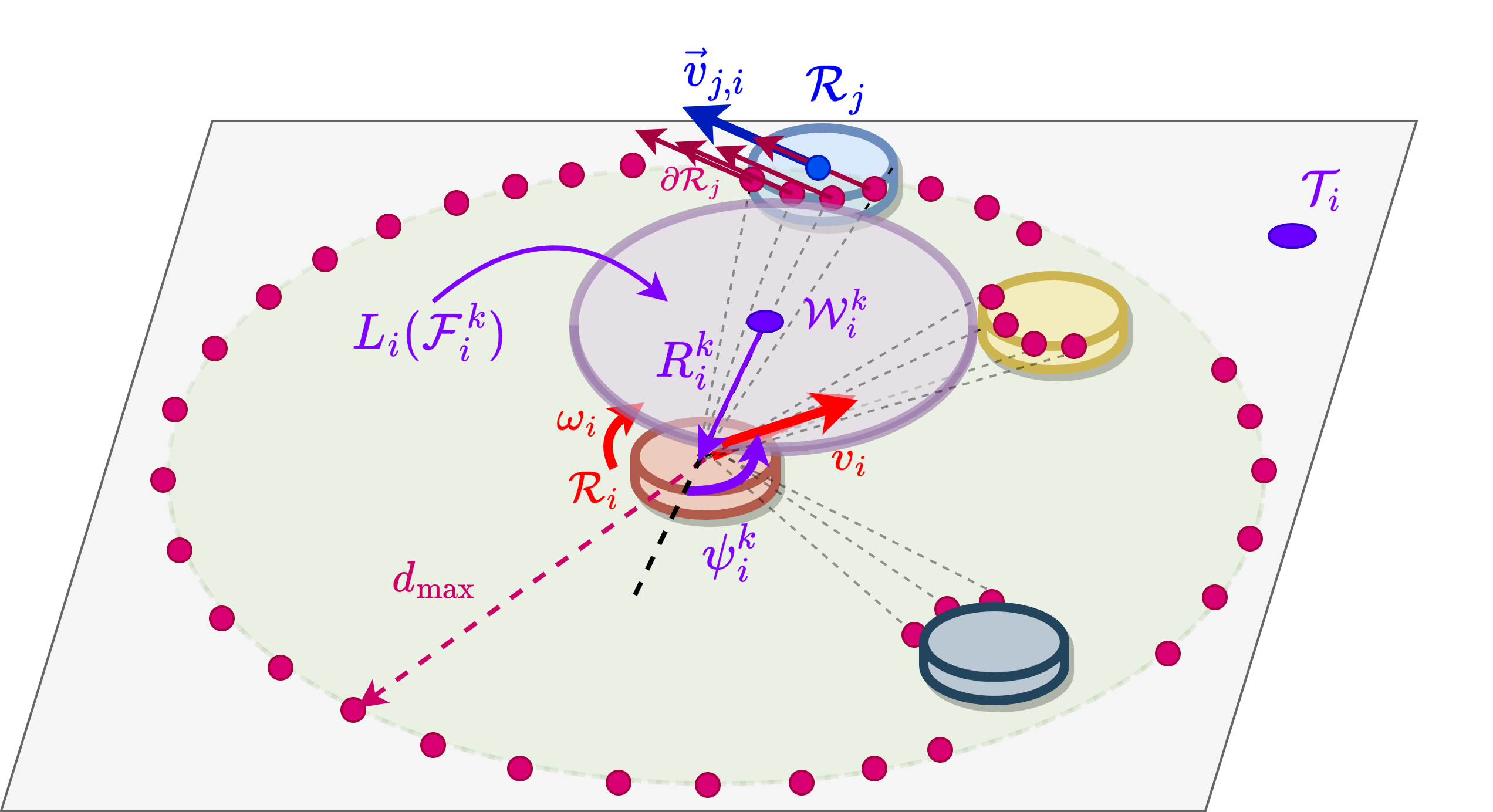}
    \caption{The Problem Formulation}
    \label{fig:ProblemFormulation}
\end{figure}

Each agent in the non-adversarial crowd is a non-holonomic mobile robot ($\mathcal{R}_i$ for each $i=1,2,...,N$) in this work, and they are modeled using the unicycle kinematics with respect to a target point $\mathcal{W}^k_i$ (identified at the $k^{\text{th}}$ instance).
\begin{align}
    \label{eqn:SystemKinematics}
    \begin{bmatrix}
        \dot{R^k_i}\\
        \dot{\psi^k_i}
    \end{bmatrix} = \begin{bmatrix}
        v_i\cos{(\psi^k_i)}\\
        \omega_i - \frac{v_i}{R^k_i}\sin{(\psi^k_i)}
    \end{bmatrix}
\end{align}
The states are its radial distance $\left(R^k_i\right)$ to $\mathcal{W}^k_i$, and relative bearing $(\psi^k_i \in (-\pi,\pi])$ is defined around the axis that originates from $\mathcal{W}^k_i$ in the direction towards ${R}^k_i$ as illustrated in Fig.~\ref{fig:ProblemFormulation}. The choice of state descriptions as $(R_i^k,\psi_i^k)$ are essentially feedback quantities described relative to $\mathcal{W}^k_i$. The control inputs are the linear ($v_i$) and angular velocities ($\omega_i$). Each $\mathcal{R}_i$ is equipped with an onboard feedback controller $\mathcal{F}_i(R^k_i,\psi^k_i;\mathcal{W}^k_i)$ that is capable of inducing system trajectories that stabilize at ${\mathcal{W}^k_i}$. Let the invariant set associated with the $\mathcal{F}_i^k(R^k_i,\psi^k_i;\mathcal{W}^k_i)$ be $L_i(\mathcal{F}_i^k)$. The mild green region represents the perception limits of the onboard sensing elements (say, LiDAR), and the pink dots represent the discrete set of measured range points - $\{P_m(\rho_m,\theta_m)\}$ in the ego-centric frame of the $\mathcal{R}_i$ that can either lie on the surface (say, $\partial \mathcal{R}_j$) of a neighboring agent or on the sensing limits ($d_{\text{max}}$) of the onboard LiDAR. The estimated velocity of the $\mathcal{R}_j$ in the ego-centric frame of the $\mathcal{R}_i$ is ${\vec{v}}_{j,i}$, and the target navigation point for $\mathcal{R}_i$ is $\mathcal{T}_i$.

In this setting, the problems addressed in this work are described as follows.
\begin{itemize}
    \item Design of an input-constrained feedback controller $\mathcal{F}_i(R^k_i,\psi^k_i;\mathcal{W}_k^i)$ that induces system trajectories such that they converge on $\mathcal{W}^k_i$, along with an explicit geometric characterization of its associated invariant set $L_i(\mathcal{F}^k_i)$ for multi-agent self and safe navigation solely through local sensing and without mutual coordination.
    \item Design of a planning strategy that iteratively identifies and induces convergence of $\mathcal{W}_k^i$ to $\mathcal{T}_i$ by successively identifying $L_i(\mathcal{F}^k_i)$ within the bounds of the safety perceived through the onboard sensing elements.
    \item We show that the proposed self-navigation algorithm provides scope for parallelization, and a demonstration of the hardware acceleration achieved through parallel processing is presented.  
\end{itemize}
\subsection{Contributions and Organization}
The key contributions of this work are organized as follows: Section \ref{sec:2} describes the design of the feedback controller and the geometry of its associated invariant set.
Section \ref{sec:3} provides a characterization to describe the invariant sets of $\mathcal{F}^k_i$ as state constraints within the sensed information obtained during run-time. Further, the proposed planning strategy to complete the description of $\mathcal{F}^k_i$ by identifying $\mathcal{W}^k_i$ to obtain the feedback control inputs for safe navigation towards $\mathcal{T}_i$ is described, and the overall self-navigation algorithm is then elaborated in this section. In Section \ref{sec:4}, the proposed algorithm is verified through MATLAB simulations, demonstrated via experiments using a set of turtlebot3 robots, and a baseline comparison is presented with a decoupled navigation strategy.  The potential of the proposed algorithm to be parallelized to accelerate real-time computations and a mathematical quantification of the acceleration achieved is presented in Section \ref{sec:5}. Finally, the potential directions of future work and conclusions are presented in Section \ref{sec:6}

\section{FEEDBACK CONTROL DESIGN}
\label{sec:2}
The objectives $(\bm{\mathcal{O}_i})$ of the feedback control design for $\mathcal{R}_i$ in this section are as follows,
\begin{itemize}
\item [$\bm{\mathcal{O}_1}$] To orient $\mathcal{R}_i$ in either parallel/anti-parallel configuration relative to the instantaneous position vector of $\mathcal{R}_i$ originating at $\mathcal{W}^k_i$. 
\item [$\bm{\mathcal{O}_2}$] To achieve convergence to $\mathcal{W}^k_i$, and explicitly characterize the geometry of $L_i(\mathcal{F}_i^k)$ associated with $\mathcal{F}_i^k$. 
\end{itemize}
The feedback controller $\mathcal{F}^k_i$ is designed in terms of the feedback states $(R^k_i,\psi^k_i)$ described relative to $\mathcal{W}^k_i$ (more details in Theorems \ref{thm:Orientation Stabilization},\ref{thm:Position Stabilization})
\begin{align}
    \label{eqn:ControlDesign}
    \begin{bmatrix}
    v_i\\
    \omega_i
    \end{bmatrix} = \begin{bmatrix}
    -K_1\tanh(R^k_i)\sgn(c(\psi^k_i))\\
    -K_2|\sigma^k_i|^{\frac{1}{2}}\sgn(\sigma^k_i) - K_1\left(\frac{\tanh(R_i^k)}{R_i^k}\right)\sgn(c(\psi^k_i))s(\psi^k_i)
    \end{bmatrix}
\end{align}
Here, $c(\cdot) = \cos(\cdot),\;s(\cdot) = \sin(\cdot),\;\sgn(\cdot)$ is the signum function, $\tanh(.)$ is the hyperbolic tangent function, and $K_1>0,K_2>0$ are the control gains (design choices). 

\begin{align}
    \label{eqn:signum function definition}
    \sgn(x) = &\begin{cases}
    1 &x>0\\
    \{\begin{matrix}
    -1 &1
    \end{matrix}\} &x=0\\
    -1 &x<0 
    \end{cases}
\end{align}
We also describe the manifold $\sigma^k_i$ based on the relative bearing ($\psi^k_i(0)$) at the instant $(t=0)$ when it begins its navigation towards $\mathcal{W}^k_i$,
\begin{align}
\label{eqn:Sigma Def}
    \sigma^k_i =  \left \{ \begin{array}{rr} \psi^k_i-\sgn(\psi^k_i)\pi , &-1 \leq\cos(\psi^k_i(0)) < 0\\
\psi^k_i, &0 \leq \cos(\psi^k_i(0)) \leq 1 \end{array} \right\}
\end{align}
Now, we show using Theorems \ref{thm:Orientation Stabilization} and \ref{thm:Position Stabilization} that the controller given by \eqref{eqn:ControlDesign} achieves the said objectives $\mathcal{O}_1$ and $\mathcal{O}_2$ for each agent $\mathcal{R}_i$ while adhering to the actuation limits.
\begin{theorem}
\label{thm:Orientation Stabilization}
The control design in $(\ref{eqn:ControlDesign})$ accomplishes $\mathcal{O}_1$, while the input $\omega_i$ is constrained within $|\omega_i|\leq K_2\left(\frac{\pi}{2}\right) + K_1$.
\end{theorem}

\begin{proof}
From the definition of $\sigma^k_i$ as in (\ref{eqn:Sigma Def}),
\begin{align*}
    \dot{\sigma}^k_i = &\dot{\psi}^k_i\\
                 = &\omega_i - \frac{v_i}{R^k_i}\sin(\psi^k_i)
\end{align*}
Now, on substituting the proposed control law we get,
\begin{align}
    \label{eqn:Super Twisting Algorithm}
    \dot{\sigma}^k_i = -K_2|\sigma^k_i|^{\frac{1}{2}}\sgn(\sigma^k_i)
\end{align}
Considering the Lyapunov function $V_{1,i}^k = \frac{1}{2}(\sigma_i^k)^2$ (positive-definite and radially unbounded) for the $\mathcal{R}_i$,
\begin{align*}
    {\dot{V}}^k_{1,i} = &\sigma^k_i\dot{\sigma}^k_i \;\;\; \mbox{for any }\; k=1,2,\ldots, N\\
            = &-K_2|\sigma^k_i|^{\frac{3}{2}}\\
    \implies {\dot{V}}^k_{1,i} = &-K_2(2V^k_{1,i})^{\frac{3}{4}}
\end{align*}
The above condition results in finite-time stability (A sketch of a similar proof can be found in \cite{IntSlidingMode}). Therefore, according to the definitions of $\sigma^k_i$ as in (\ref{eqn:Sigma Def}), $\psi^k_i$ attains an equilibrium on the set $\{0,\pi\}$. (Considering the representation $\sigma^k_i \in (-\pi,\pi]$) 

The bounds on the proposed control input $\omega_i$ in (\ref{eqn:ControlDesign}) can be shown as follows, (using the triangular inequality)
\begin{align*}
    |\omega_i| \leq & \underbrace{\left|-K_2|\sigma^k_i|^{\frac{1}{2}}\sgn(\sigma^k_i)\right|}_{T_1} \\ &\quad +  \underbrace{\left|K_1\left(\frac{\tanh(R_i^k)}{R_i^k}\right)\sgn(c(\psi^k_i))s(\psi^k_i)\right|}_{T_2}
\end{align*}
The terms $(T_1,T_2)$ are bounded as follows,
\begin{align*}
    T_1 \leq &K_2\left(\frac{\pi}{2}\right) \qquad \because |\sigma_i^k| \leq \frac{\pi}{2} \\
    T_2 \leq K_1 \qquad &\because \left(\frac{\tanh(R_i^k)}{R_i^k}\right)\leq 1, s(\psi^k_i)\leq 1
\end{align*}
Therefore, the bound on control input $\omega_i$ is obtained as,
\begin{align*}
    |\omega_i| \leq &K_2\left(\frac{\pi}{2}\right) + K_1
\end{align*}
The control gains $(K_1,K_2)$ are chosen such that $\omega_i$ does not violate the available input actuator constraints. 
\end{proof}
\begin{theorem}
\label{thm:Position Stabilization}
The control design in $(\ref{eqn:ControlDesign})$ accomplishes $\mathcal{O}_2$ while ensuring both the input $|v_i|\leq K_1$, and the position trajectories are confined within an explicit invariant region $L_i(\mathcal{F}_i^k)$ as stated.
\end{theorem}

\begin{proof}
    From (\ref{eqn:Sigma Def}), the case of $\psi^k_i= \left\{-\frac{\pi}{2},\frac{\pi}{2}\right\}$ is not a fixed point of the system kinematics since the control design ensures that $\psi^k_i$ always converges to a set of points - $\{0,\pi\}$ (Theorem \ref{thm:Orientation Stabilization}). The stability of the proposed control design in accomplishing $\mathcal{O}_2$ is shown using the Lyapunov function $V_{2,i}^k$ for the $\mathcal{R}_i$,
\begin{align*}
    V_{2,i}^k = &\frac{1}{2}(R^k_i)^2 \;\;\; \mbox{for any }\; k=1,2,\ldots, N\\
    \implies {\dot{V}}^k_{2,i} = &R^k_i\dot{R}^k_i\\
                     = &-K_1(R^k_i)\cdot\tanh{(R^k_i)}\cdot|\cos(\psi^k_i)|\\
    \implies {\dot{V}}^k_{2,i} = &\left(-K_1|\cos(\psi^k_i)|\right)(R^k_i)\cdot\tanh{(R^k_i)}\\
    \implies {\dot{V}}^k_{2,i} < &0 \quad \forall \psi^k_i \neq \left\{-\frac{\pi}{2},\frac{\pi}{2}\right\}
    \end{align*}
    Therefore, the proposed control design results in a monotonic decrease of $R^k_i$. Subsequently, as $R^k_i\rightarrow0$, $\tanh(R^k_i)\sim R^k_i$. The following then holds true,
    \begin{align*}
    {\dot{V}}^k_{2,i} = &\left(-K_1|\cos(\psi^k_i)|\right) (R^k_i)^2\\
    \implies {\dot{V}}^k_{2,i} = &-\gamma^k_i(t) V^k_{2,i} \quad;\quad \gamma^k_i(t) = K_1|\cos(\psi^k_i(t))|
\end{align*}
The rate of decay $\gamma^k_i(t)$ of $V^k_{2,i}$ has the following characteristics (from Theorem \ref{thm:Orientation Stabilization})
\begin{itemize}
    \item $\gamma^k_i(t) > 0$ as $\psi^k_i= \left\{-\frac{\pi}{2},\frac{\pi}{2}\right\}$ are not fixed points.
    \item $\gamma^k_i(t) \rightarrow 2K_1$ as $\psi^k_i \rightarrow\{0,\pi\}$, and in finite time.
\end{itemize}
Therefore, the control design ensures asymptotic convergence (position trajectories) as the agent approaches $\mathcal{W}^k_i$. It's also trivial that,
\begin{align*}
    |v_i| \leq K_1 \qquad \because \tanh{(R^k_i)}\leq 1
\end{align*}
Therefore, the control gain $K_1$ is chosen such that $v_i$ does not violate the input constraints of the system.

The Lyapunov condition for stability also guarantees that the sub-level sets of the Lyapunov function are rendered positively invariant during stabilization. 
Mathematically, given a function $f:\mathbb{R}^n \rightarrow \mathbb{R}$, its sub-level set $L$ is defined as follows,
\begin{align*}
    L = \{X\in \mathbb{R}^n : f(X) \leq K \} \quad;\quad K \in \mathbb{R}
\end{align*}
Therefore, the agent's position is constrained within the following set (invariant set) during stabilization to $\mathcal{W}^k_i$.
\begin{align}
    L_i(\mathcal{F}_i^k) := &\left\{R^k_i\in \mathbb{R}^{+} \cup \{0\} : V_2(R^k_i) \leq V_2(R^k_i(0)) \right\}\\
                        \label{eqn:InvRegion}
                    := &\left\{R^k_i\in \mathbb{R}^{+} \cup \{0\} : R^k_i \leq R^k_i(0) \right\}
\end{align}
Here, $R^k_i(0)$ is the radial distance of the $\mathcal{R}_i$ to $\mathcal{W}^k_i$ at the time instant when it begins its navigation towards $\mathcal{W}^k_i$ when $\mathcal{F}_i^k$ is invoked.
From (\ref{eqn:InvRegion}), \textbf{the nature of the sub-level set $L_i(\mathcal{F}^k_i)$ is a disc} centered at $\mathcal{W}^k_i$ with a radius of $R^k_i(0)$ in the $\mathcal{R}_i$'s vicinity. It is also trivial to note that the agent always starts navigating from the boundary of $L_i(\mathcal{F}^k_i)$ (say $\partial L^k_i$) during stabilization. (Since, $\partial L^k_i = \left\{R^k_i\in \mathbb{R}^{+} \cup \{0\} : V^k_{2,i}(R^k_i) = V^k_{2,i}(R^k_i(0)) \right\}$ and by definition, when the agent starts navigating : $R^k_i=R^k_i(0) \implies V_{2,i}^k(R^k)=V_{2,i}^k(R^k_i(0))$ always holds.)
\end{proof}
Since the individual states of the system - $(\mathcal{R}_i^k,\mathcal{\psi}_i^k)$ themselves exhibit Lyapunov stability through the Lyapunov functions $(V^k_{1,i},V^k_{2,i})$, the overall system is also stable in the sense of Lyapunov $(V^k_{1,i}+V^k_{2,i})$.
\begin{align*}
    \dot{V}_{1,i}^k < 0, \dot{V}_{2,i}^k < 0 \implies (\dot{V}_{1,i}^k+, \dot{V}_{2,i}^k) < 0 
\end{align*}

\section{THE PROPOSED SELF-NAVIGATOR}
\label{sec:3}

\begin{table}[h]
    \centering
    \caption{Notations in Section \ref{sec:3}}
    \label{tab:notations}
    \begin{tabular}{|c|p{5cm}|}
        \hline
        \textbf{Notation} & \textbf{Description} \\
        \hline\hline
        $\mathcal{R}_i$ & The $i^{\text{th}}$ agent\\
        \hline
        $\partial \mathcal{R}_i$ & Boundary of the shape of $i^{\text{th}}$ agent\\
        \hline
        $\vec{v}_{j,i}$ & Relative velocity of $\mathcal{R}_j$ relative to $\mathcal{R}_i$.\\
        \hline
        $\mathcal{M}_i(\{\rho_m,\theta_m,\vec{v}_m\})$& Measurement tuples sensed onboard $\mathcal{R}_i$. Each tuple corresponds to the $m^{\text{th}}$ measurement, and consists of the information about range $(\rho_m)$, the angle $(\theta_m)$ at which it was obtained, and its relative velocity $(\vec{v}_m)$.\\
        \hline 
        $d_{\vec{v}}(P,L)$ & Distance between a geometric shape $L$ and an external point P along the direction vector $\vec{v}$\\
        \hline
        $d_{\min}(P,L)$ & Shortest distance between geometric shape $L$ and an external point P\\
        \hline 
        $N$ & Total number of measurements\\
        \hline\hline
    \end{tabular}
\end{table}

As $\mathcal{R}_i$ navigates amidst the vicinity of other agents, at each planning instance, a correlation is obtained between the sensed range data points - $\{P_m(\rho_m,\theta_m)\}$ and the corresponding moving agents ($\mathcal{R}_j$) in its vicinity. The correlated range data points are then assigned a velocity ($\vec{v}_m$) as follows,
\begin{align*}
    \vec{v}_m = &\begin{cases}
    \vec{v}_{j,i}, &\text{if }P_m(\rho_m,\theta_m) \in \partial \mathcal{R}_j\\
    \vec{0}, &\text{otherwise} 
    \end{cases}
\end{align*}
The bounds of safety are then characterized through the constructed set of measurement tuples - $\mathcal{M}_i\{(\rho_m,\theta_m,\vec{v}_m)\}$.

\subsection{Identifying the set of candidate $L_i(\mathcal{F}_i^k)$ around $\mathcal{R}_i$}\label{sec:3A}



As the geometric structure of $L_i(\mathcal{F}_i^k)$ is a disc, each candidate $L_i(\mathcal{F}_i^k)$ can be uniquely parameterized through its center and radius - $(\mathcal{W},d)$. 

Since $\mathcal{R}_i$ starts navigating from the boundary $\partial L_i(\mathcal{F}_i^k)$ of $L_i(\mathcal{F}_i^k)$ (Theorem \ref{thm:Position Stabilization}), the parametric representation of a candidate $L_i(\mathcal{F}_i^k)$ is sought with $\mathcal{R}_i \in \partial L_i(\mathcal{F}_i^k)$. Now, to characterize the set of candidate $L_i(\mathcal{F}_i^k)$ in the ego-centric frame of the $\mathcal{R}_i$, the angular space $(\theta_n)$ is discretized in coherence with the angles $\theta_m$ at which $\rho_m$ are measured. Along a candidate direction $\theta_n$, the center $\mathcal{W}$ of any candidate $L_i(\mathcal{F}_i^k)$ whose radius is $'d'$ is at a distance $'d'$ units along $\theta_n$.

The maximum distance at which $\mathcal{W}$ could be selected for describing the feedback controller $\mathcal{F}^k_i$ for safe navigation (associated $L_i(\mathcal{F}^k_i)$ will be free of encroachment by the neighboring agents) is computed by describing the function $\mathcal{D}^k_i:\theta_n\rightarrow\mathbb{R}$ as follows $(\forall \theta_n \in \{\theta_m\})$,
\begin{align}
    \mathcal{D}^k_i(\theta_n) =\;&\underset{d}{\text{argmax}} \quad d \label{eqn:Opti_Def}\\
    \text{subject to} \quad&\mathcal{C}_m, \quad \forall m \in \{1,2,...,N\}
\end{align}
The constraints $\mathcal{C}_m$ are described as follows,
\begin{align}
    \mathcal{C}_m=\begin{cases}
    d_{\vec{v}_m}(P_m,L_i^k) > \frac{|\vec{v}_m|}{f_p}, &\text{if } \vec{v}_m \neq \vec{0}\\
    d_{\min}(P_m,L_i^k) > 0, &$\text{otherwise}$
    \end{cases}
    \label{eqn:velocity_consts}
\end{align}
\begin{remark}
Since the geometric structure of $L^k_i$ is well-defined, both $d_{\vec{v}_m}(P_m,L_i^k),d_{\min}(P_m,L_i^k))$ exhibit closed-form solutions as functions of $(d,\theta_n)$. Also, as $L^k_i$ is a disc, $d_{\vec{v}_m}(P_m,L_i^k) = d_{\min}(P_m,L_i^k))$ holds true when the direction of $\vec{v}_m$ coincides with the direction of the ray emanating from $P_m(\rho_m,\theta_m)$ and terminating at $\mathcal{W}$ (geometric center of $L^k_i$).
\end{remark}


\begin{figure}
    \centering
    \includegraphics[width = \columnwidth]{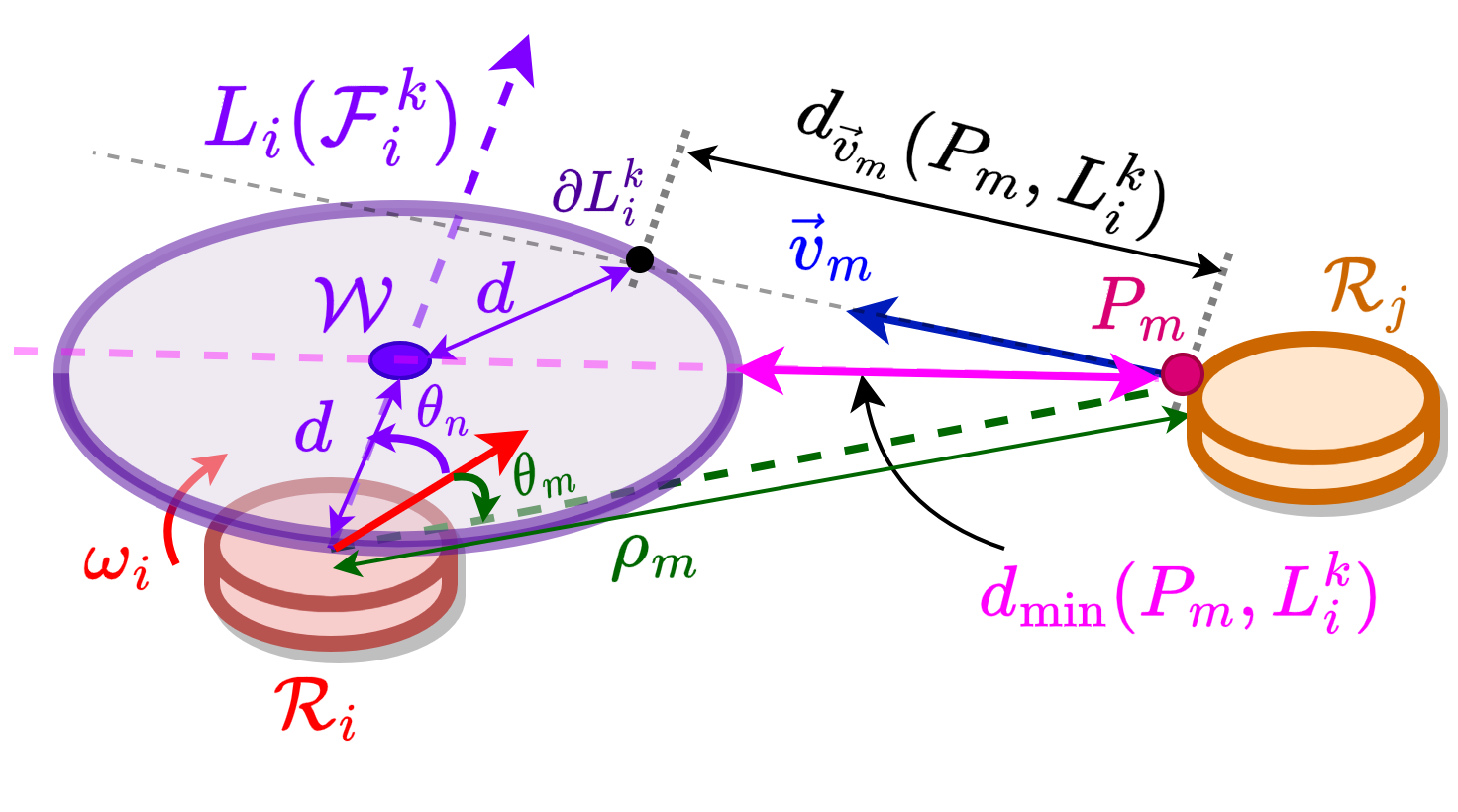}
    \caption{Describing the invariant set ${L}^k_i$ associated with $\mathcal{F}^k_i$ within the bounds of safety perceived during runtime.}
    \label{fig:Gen_Engage_Scenario}
\end{figure}

To compute safe control inputs at the $k^{\text{th}}$ planning instance (say, at time $t=t_k$), the constraints in \eqref{eqn:velocity_consts} ensure that $\mathcal{M}_i$ do not encroach into any candidate $L_i(\mathcal{F}_i^k)$ until the subsequent $(k+1)^{\text{th}}$ planning instance. 

By considering a constant velocity model when $\vec{v}_m\neq0$ between consecutive planning instances, the constraint \eqref{eqn:velocity_consts} enforces that the relative displacements of $P_m$ are strictly less than $d_{\vec{v}_m}(P_m,L^k_i)$ to prevent encroachment into $L_i(\mathcal{F}_i^k)$. When $P_m$ is relatively stationary ($\vec{v}_m=0$), the shortest between $P_m$ and $L_i^k$ is considered to prevent encroachment.  

Since $L_i(\mathcal{F}_i^k)$ is the invariant set (convex) within which the naturally induced system trajectories evolve when the control inputs computed from associated $\mathcal{F}_i^k$ are deployed, the constraint in \eqref{eqn:velocity_consts} essentially guarantees that the following holds true between consecutive planning instances.
\begin{align*}
     P_m(\tau)\cap L_i(\mathcal{F}_i^k) = \emptyset, \; \forall \tau \in \left[t_k,t_k+\left(\frac{1}{f_p}\right)\right], \forall P_m\in\partial\mathcal{R}_j
\end{align*} 

\begin{remark}
    In a generic operating scenario, the precise planning instances at which every agent computes control inputs for itself for safe navigation are not synchronized. The agents are also oblivious to the exact target navigation points of their counterparts. Consequently, although the control strategy is identical for every agent in the crowd here, it is impossible for a given agent (say, $\mathcal{R}_i$) to know precisely how $\vec{v}_{j, i}$ of its neighbors is going to evolve between consecutive planning instances. This motivates the constant velocity consideration to identify candidate $L_i(\mathcal{F}_i^k)$ for computing safe control inputs for navigation between consecutive planning instances. Further, when the time interval between consecutive planning instances is small, this consideration is not detrimental for practical purposes. 
\end{remark}

\subsection{The Planning Strategy}\label{sec:3B}

The function $\mathcal{D}^k_i$ is computed at each planning instance based on $\mathcal{M}_i$. A visualization of $\mathcal{D}_i^k$ along with a candidate ${L}_i(\mathcal{F}_i^k)$ (centered at $\mathcal{W}$) is illustrated in Fig.~\ref{fig:planning_strategy}. The red dots represent the entries of $\mathcal{D}_i^k$ computed using \eqref{eqn:Opti_Def} in section \ref{sec:3A}.

Now, given the target point of navigation $\mathcal{T}_i$ for $\mathcal{R}_i$, the target point $\mathcal{W}^k_i$ at each planning instance for describing the feedback controller $\mathcal{F}^k_i$ is selected through the following strategy. The feedback control inputs in \eqref{eqn:ControlDesign} for safe navigation are computed through states described relative to $\mathcal{W}_k^i$ for $\mathcal{F}^k$. 
 \begin{align}
     \mathcal{W}^k_i = &\underset{\mathcal{W}}{\text{argmin}} \quad ||\mathcal{W} - \mathcal{T}_i|| \label{eqn:waypoint_strategy_eqn}\\
     \text{subject to}\quad &\mathcal{W} \in \{(d,\theta_n)\;:\;d \leq \mathcal{D}^k_i(\theta_n) \} \label{eqn:waypoint_strategy_consts}
 \end{align}

At each planning instance, the proposed strategy in \eqref{eqn:waypoint_strategy_eqn} greedily seeks $\mathcal{W}^k_i$ such that it is as close as $\mathcal{T}_i$ as possible within the constrained set described by $\mathcal{D}^k_i$. Since $\mathcal{F}^k_i$ is designed such that it induces trajectories towards $\mathcal{W}^k_i$, the constrained set in \eqref{eqn:waypoint_strategy_consts} also iteratively grows closer towards $\mathcal{T}_i$ between consecutive planning instances. As the objective function defined in \eqref{eqn:waypoint_strategy_eqn} has its minima at $\mathcal{W} = \mathcal{T}_i$, once this constrained set grows sufficiently close to $\mathcal{T}_i$ such that it encompasses it at an arbitrary planning instance, the proposed planning strategy achieves a latch of $\mathcal{W}^k_i$ onto $\mathcal{T}_i$. Consequently, the resulting trajectories induced by $\mathcal{F}^k_i$ safely converge on $\mathcal{T}_i$ once $\mathcal{W}^k_i = \mathcal{T}_i$ is achieved. 

\begin{remark}
    The choice of the proposed strategy in \eqref{eqn:waypoint_strategy_eqn} under the constraints in \eqref{eqn:waypoint_strategy_consts} is not restrictive, and other alternative intelligent formulations could be explored as long as they could iteratively induce convergence of $\mathcal{W}^k_i$ to $\mathcal{T}_i$.
\end{remark}

\begin{figure}
    \centering
    \includegraphics[width = 0.8\columnwidth]{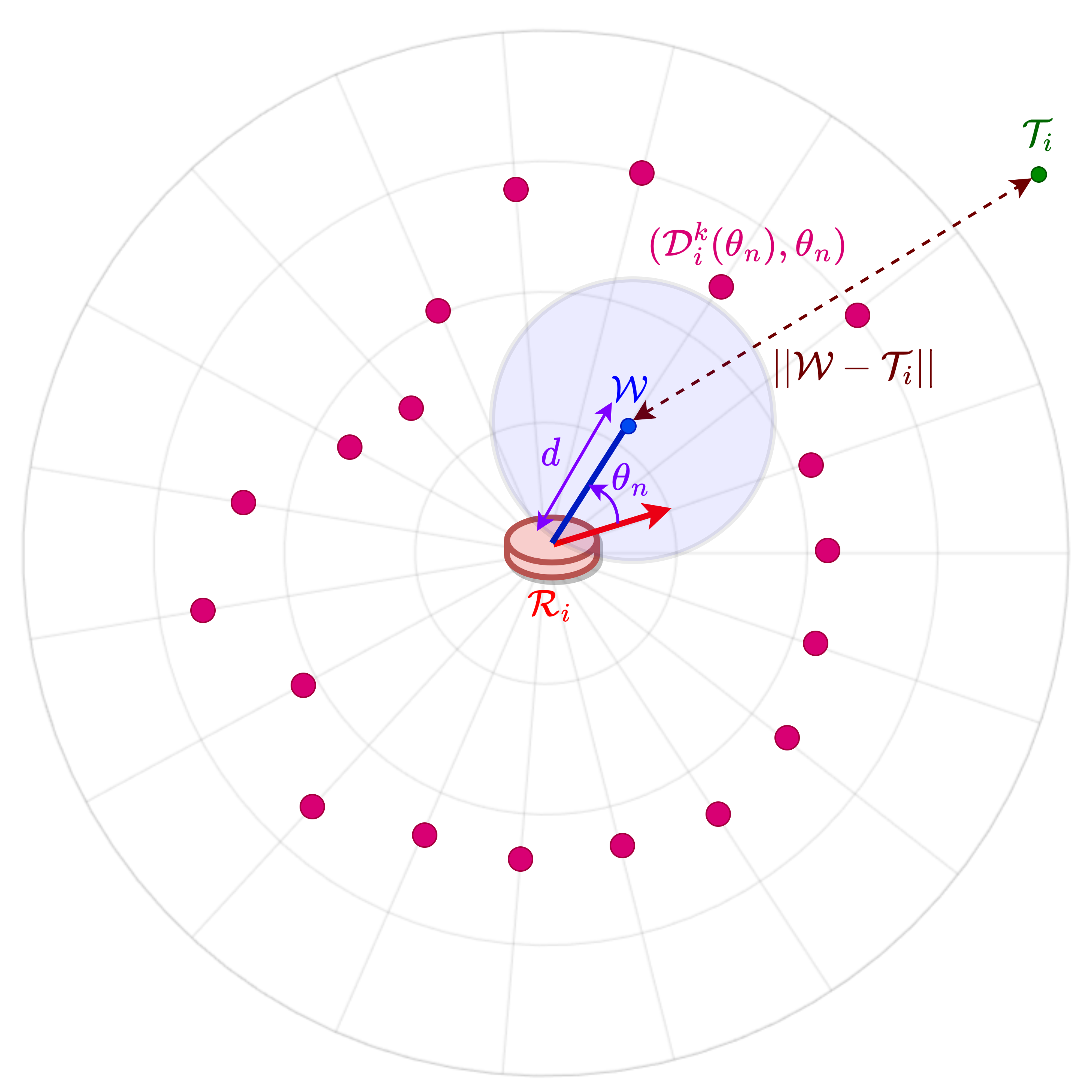}
    \caption{The Planning Strategy. A visualization of $\mathcal{D}^k_i$ at each planning instance in the ego-centric frame of the $\mathcal{R}_i$.}
    \label{fig:planning_strategy}
\end{figure}
 
\subsection{The self-navigation algorithm}\label{sec:3C}
The proposed self-navigation algorithm for each agent $\mathcal{R}_i$ is described in Algorithm \ref{alg:Proposed algorithm}. Given the target $\mathcal{T}_i$ for $\mathcal{R}_i$, and the bandwidth limitations of its onboard computer $f_p$, the proposed algorithm (Algorithm \ref{alg:Proposed algorithm}) computes feedback control inputs through a pre-synthesized feedback controller for safe navigation by exploiting the geometric structure of its associated invariant set $L_i(\mathcal{F}^k_i)$, and $\mathcal{M}_i$ sensed \& constructed during run-time. 

\begin{algorithm}
 \DontPrintSemicolon
 \KwInput{$\mathcal{T}_i,f_p,,L_i(\mathcal{F}^k_i)$}
 \KwOutput{$\mathcal{F}^k_i(R^k_i,\psi^k_i;\mathcal{W}^k_i)$}
 {
 \While{$\mathcal{R}_i\neq\mathcal{T}_i$}{
 \tcc{Loop runs at $f_p$ Hz}
\textbf{sense \& construct} $\rightarrow\{\mathcal{M}_i(\rho_m,\theta_m,\vec{v}_{m})\}$\;
\For{\textbf{each} $\theta_n$}{
 \textbf{compute} $\mathcal{D}^k_i\left(\theta_n\right)$ via (\ref{eqn:Opti_Def}) using $\left\{(\rho_m,\theta_m,\vec{v}_m)\right\}$, $L_i(\mathcal{F}^k_i)$\;
 }
 \textbf{compute} $\mathcal{W}^k_i$ via strategy in (\ref{eqn:waypoint_strategy_eqn}) using $\mathcal{D}^k_i, \mathcal{T}_i$\;
\textbf{compute} $(R^k_i,\psi^k_i)$ relative to $\mathcal{W}^k_i$\;
\textbf{compute} $\mathcal{F}^k_i(R^k_i,\psi^k_i;\mathcal{W}^k_i)$ from (\ref{eqn:ControlDesign})\;
\textbf{deploy} $\mathcal{F}^k_i(R^k_i,\psi^k_i;\mathcal{W}^k_i)$ on $\mathcal{R}_i$\;
$k\gets k+1$
}
}
\caption{Safe navigation via invariant sets induced by feedback control - $\mathcal{R}_i$}
\label{alg:Proposed algorithm}
\end{algorithm}
In line 2, the tuple consisting of the local range information is defined through the onboard sensing elements and perception algorithms. Once this tuple is available, the function $\mathcal{D}^k_i$ is constructed in lines 3-4. $\mathcal{D}^k_i$ characterizes  the region in the ego-centric frame of $\mathcal{R}_i$ for identifying $\mathcal{W}^k_i$ which describes $\mathcal{F}^k_i$. Line 5 ensures that the selection of $\mathcal{W}^k_i$ is streamlined through the proposed strategy for inducing convergence onto $\mathcal{T}_i$. Once the description of the $\mathcal{F}^k_i$ is complete through $\mathcal{W}^k_i$, the feedback control inputs are then computed via $\mathcal{F}^k_i$ and deployed on the $\mathcal{R}_i$ in lines 6-8.

\section{RESULTS \& ANALYSIS}
\label{sec:4}

\subsection{Hardware Experiments}
\label{sec:4C}

\begin{figure}
    \centering
    \includegraphics[width = \columnwidth]{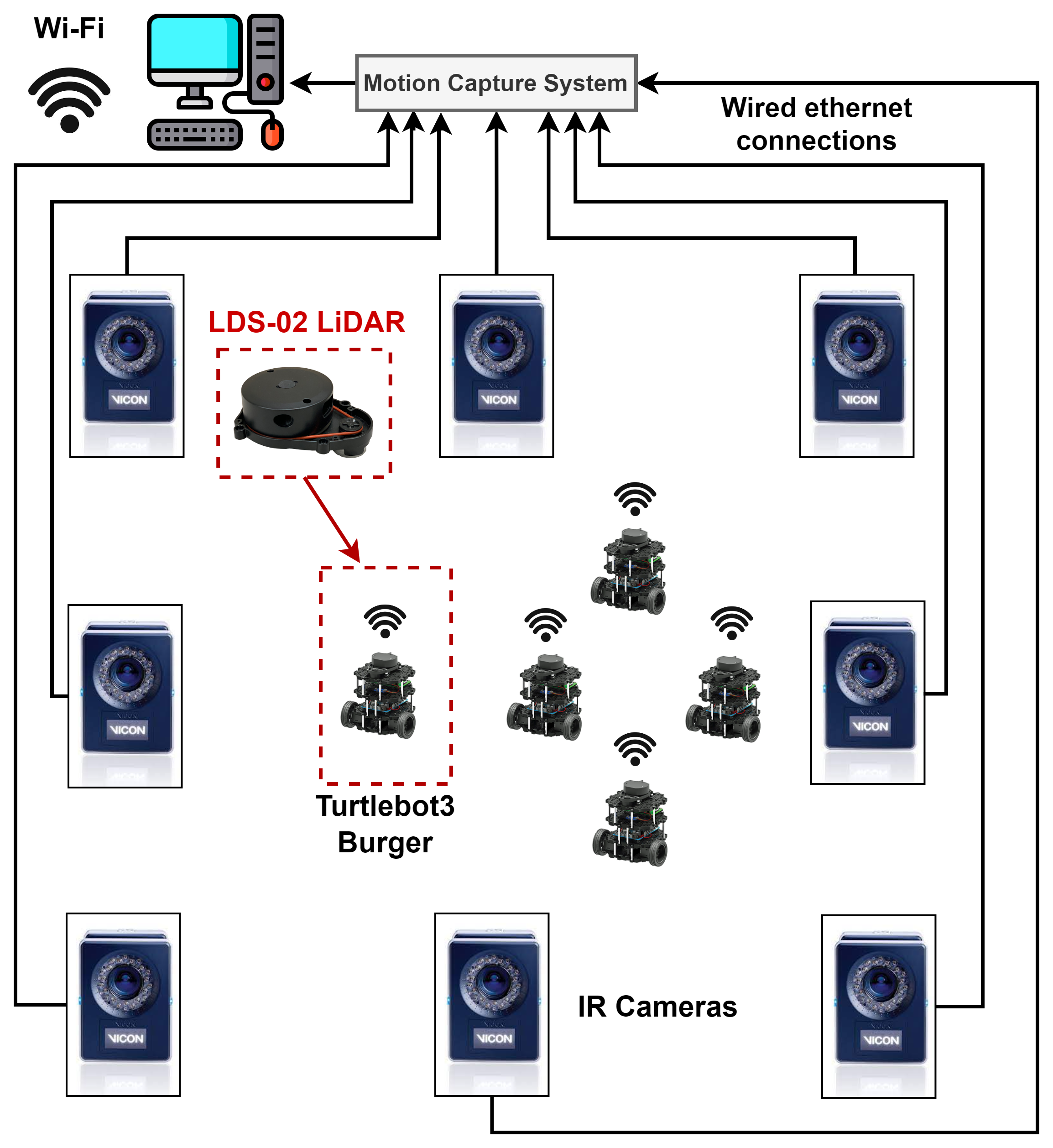}
    \caption{The experimental setup - Schematic}
    \label{fig:ExperimentSchematic}
\end{figure}

\begin{figure*}
    \centering
    \begin{subfigure}[b]{0.66\columnwidth}
    \includegraphics[width = \columnwidth]{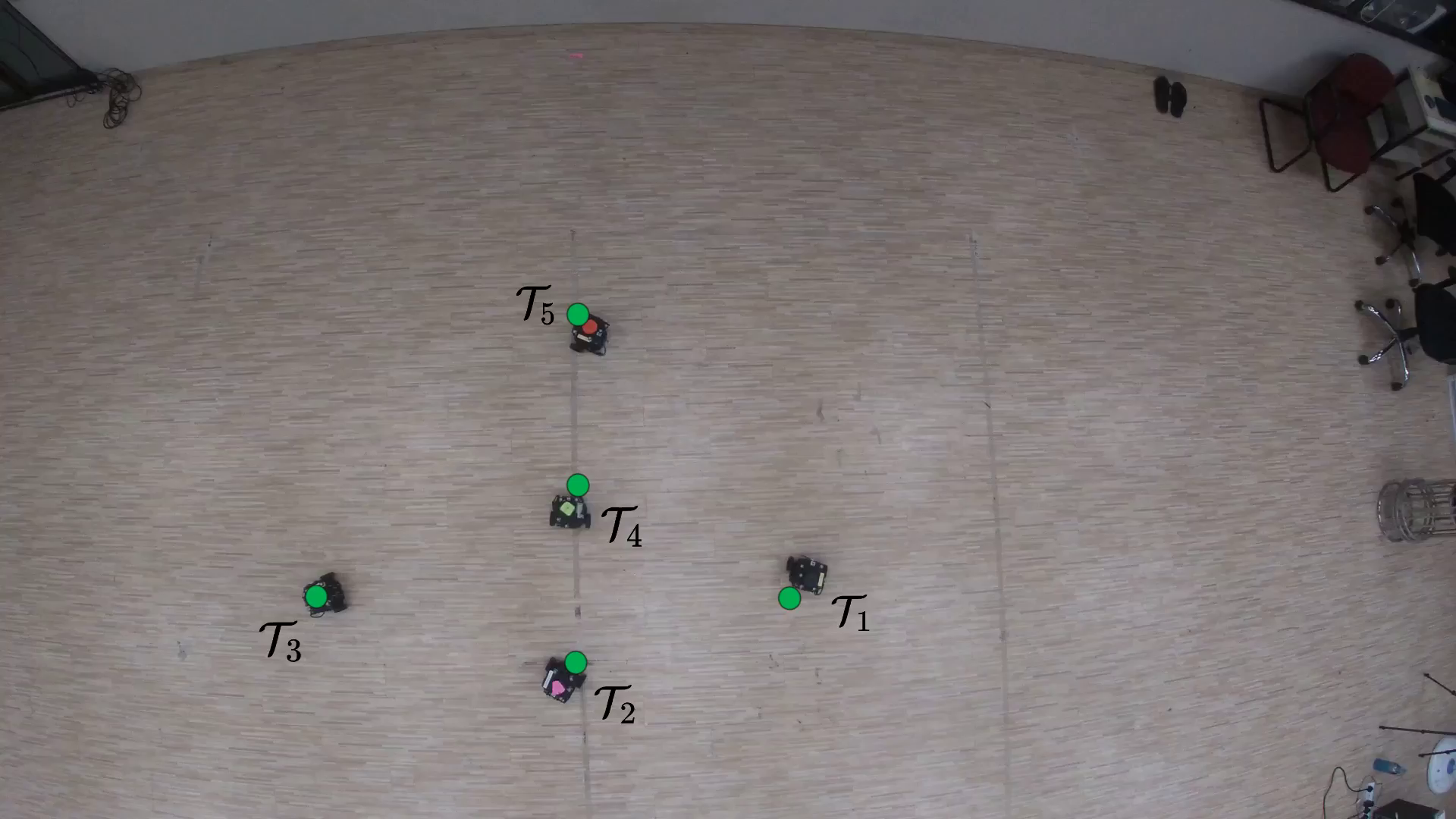}
    \caption{\textbf{T = 0 seconds}. The initial positions of the agents in the experimental setup. The green markers are reshuffled $\mathcal{T}_i$ for navigation.}
    \label{fig:00_00}
    \end{subfigure}
        \begin{subfigure}[b]{0.66\columnwidth}
    \includegraphics[width = \columnwidth]{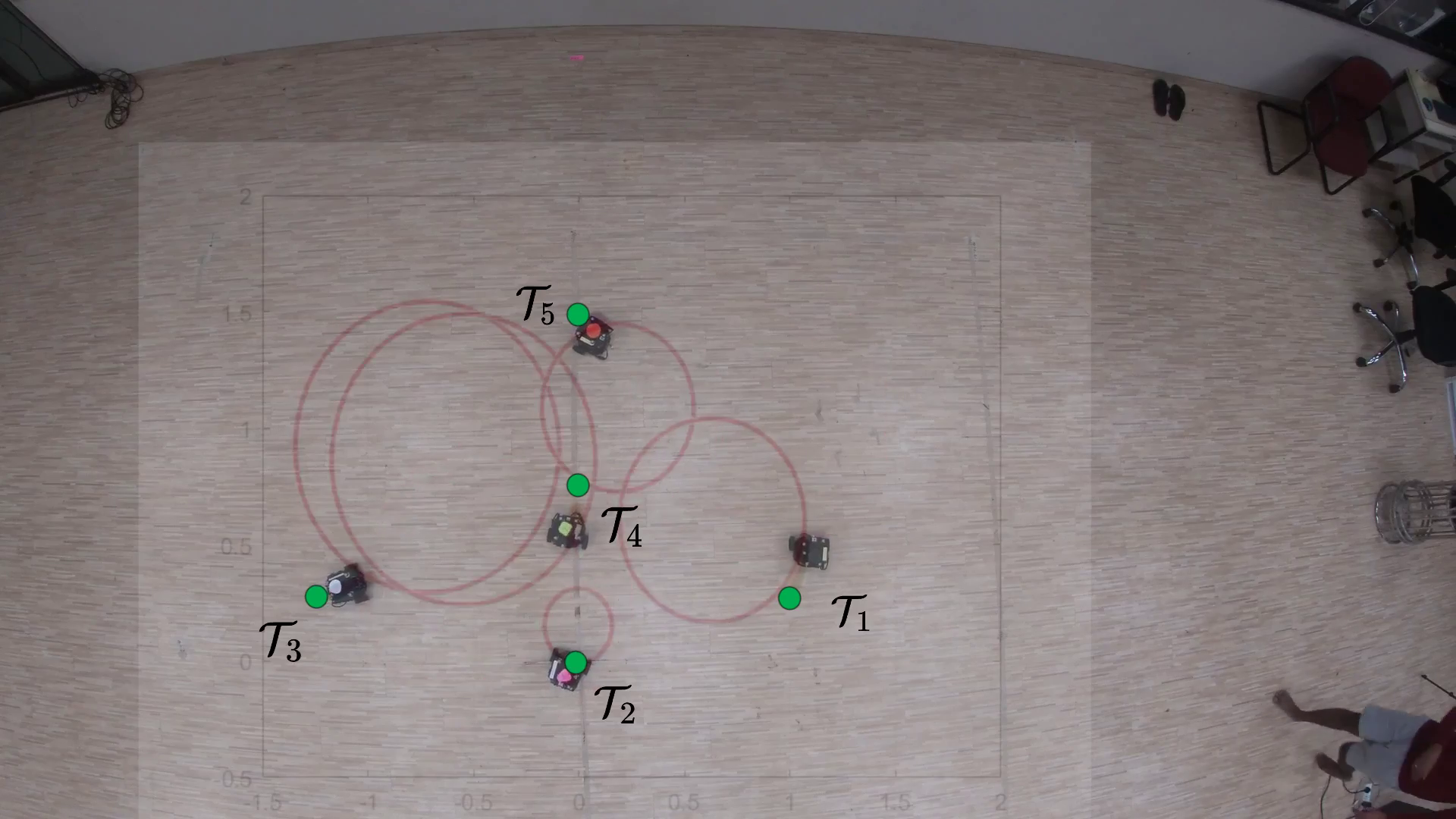}
    \caption{\textbf{T = 2 seconds}. agents identify ${L}_i(\mathcal{F}_i^k)$ \textbf{(red circles)} to compute their safe control commands from the associated onboard $\mathcal{F}_i^k$.}
    \label{fig:00_02}
    \end{subfigure}
        \begin{subfigure}[b]{0.66\columnwidth}
    \includegraphics[width = \columnwidth]{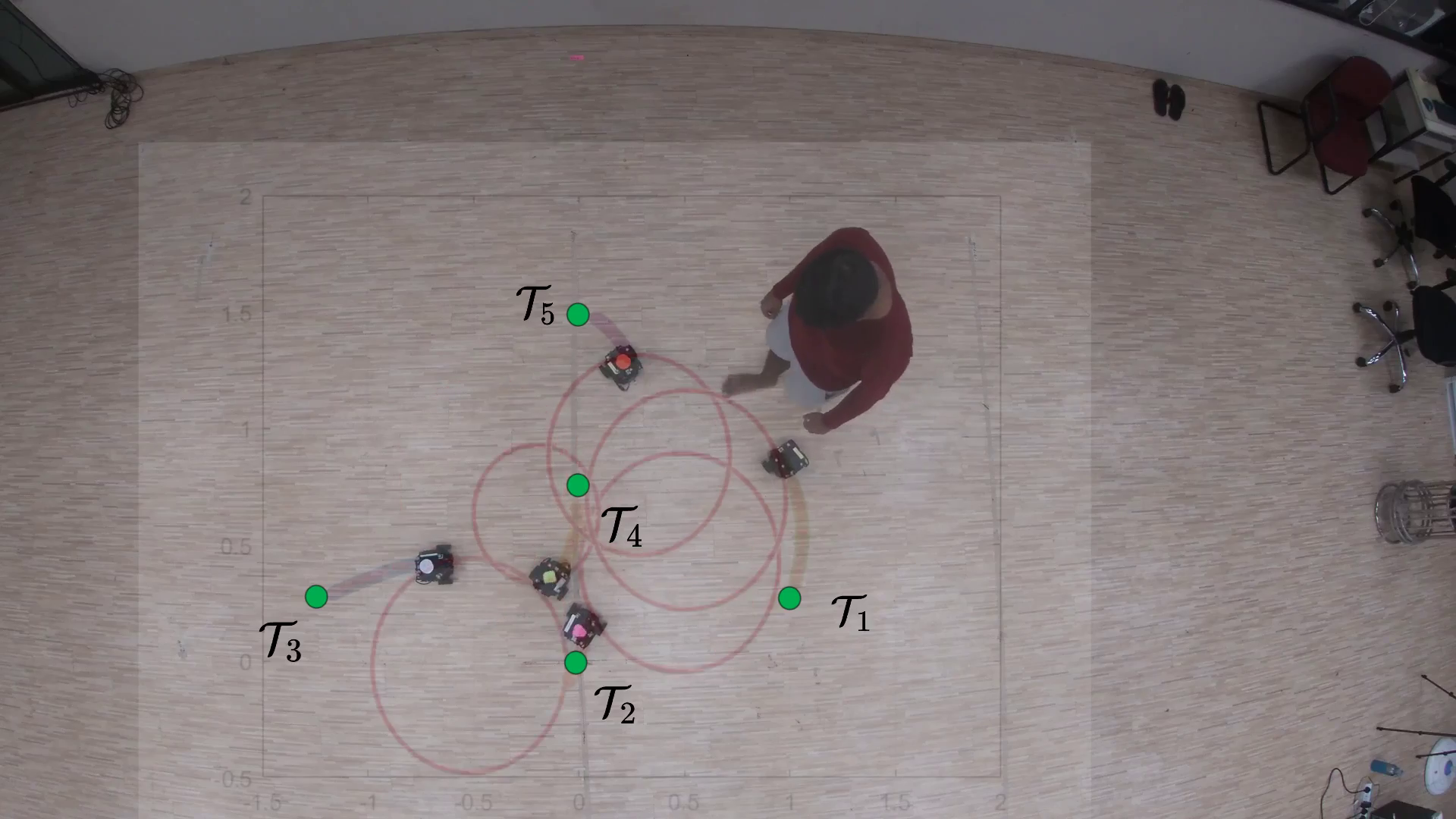}
    \caption{\textbf{T = 6 seconds}. A human influence affects the selected ${L}_i(\mathcal{F}_i^k)$ for safe control command computations for navigation.}
    \label{fig:00_06}
    \end{subfigure}
    \begin{subfigure}[b]{0.66\columnwidth}
    \includegraphics[width = \columnwidth]{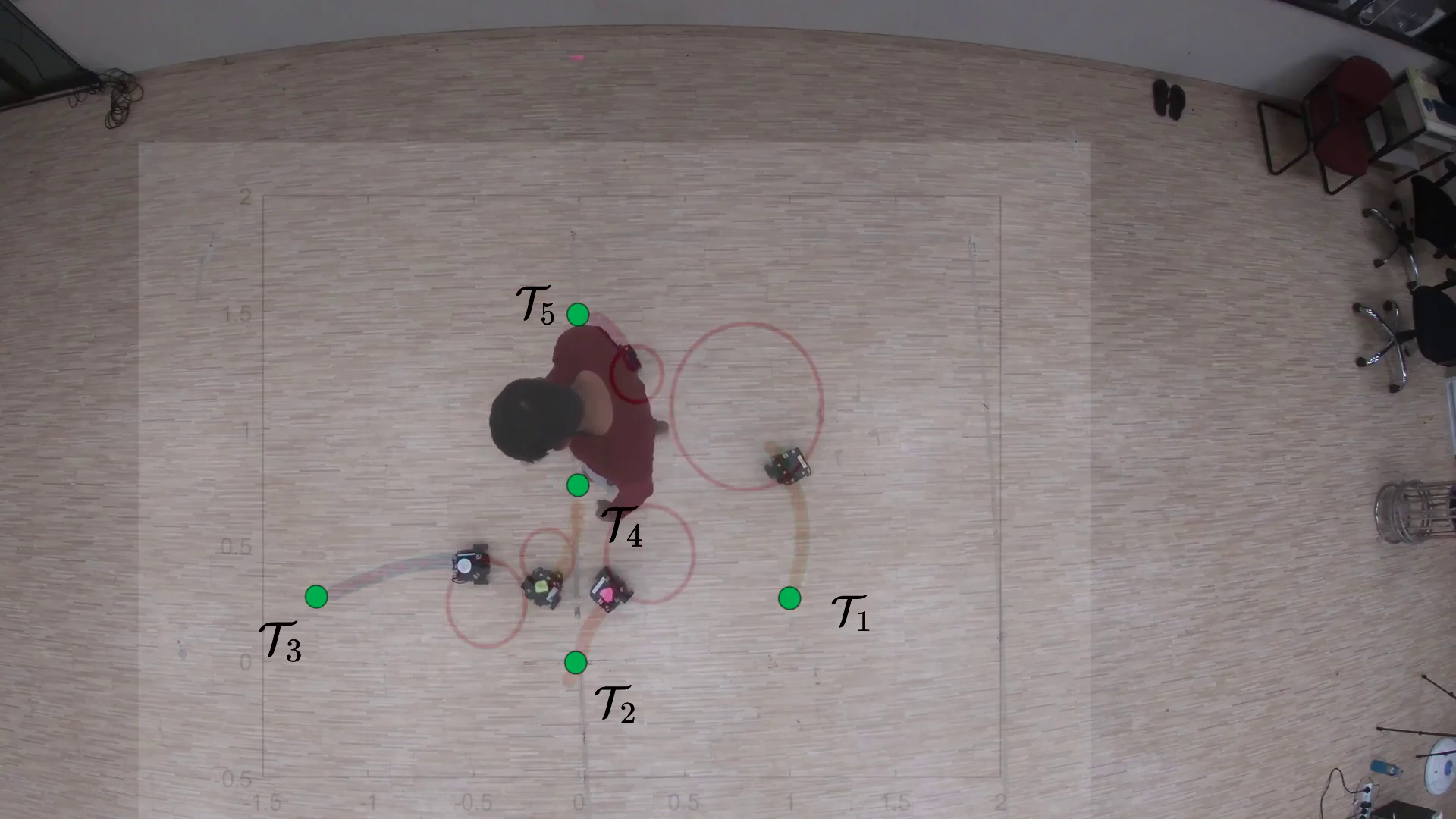}
    \caption{\textbf{T = 8 seconds}. ${L}_i(\mathcal{F}_i^k)$ of all agents shrink to avoid human encroachment (collision avoidance) within the workspace.}
    \label{fig:00_08}
    \end{subfigure}
    \begin{subfigure}[b]{0.66\columnwidth}
    \includegraphics[width = \columnwidth]{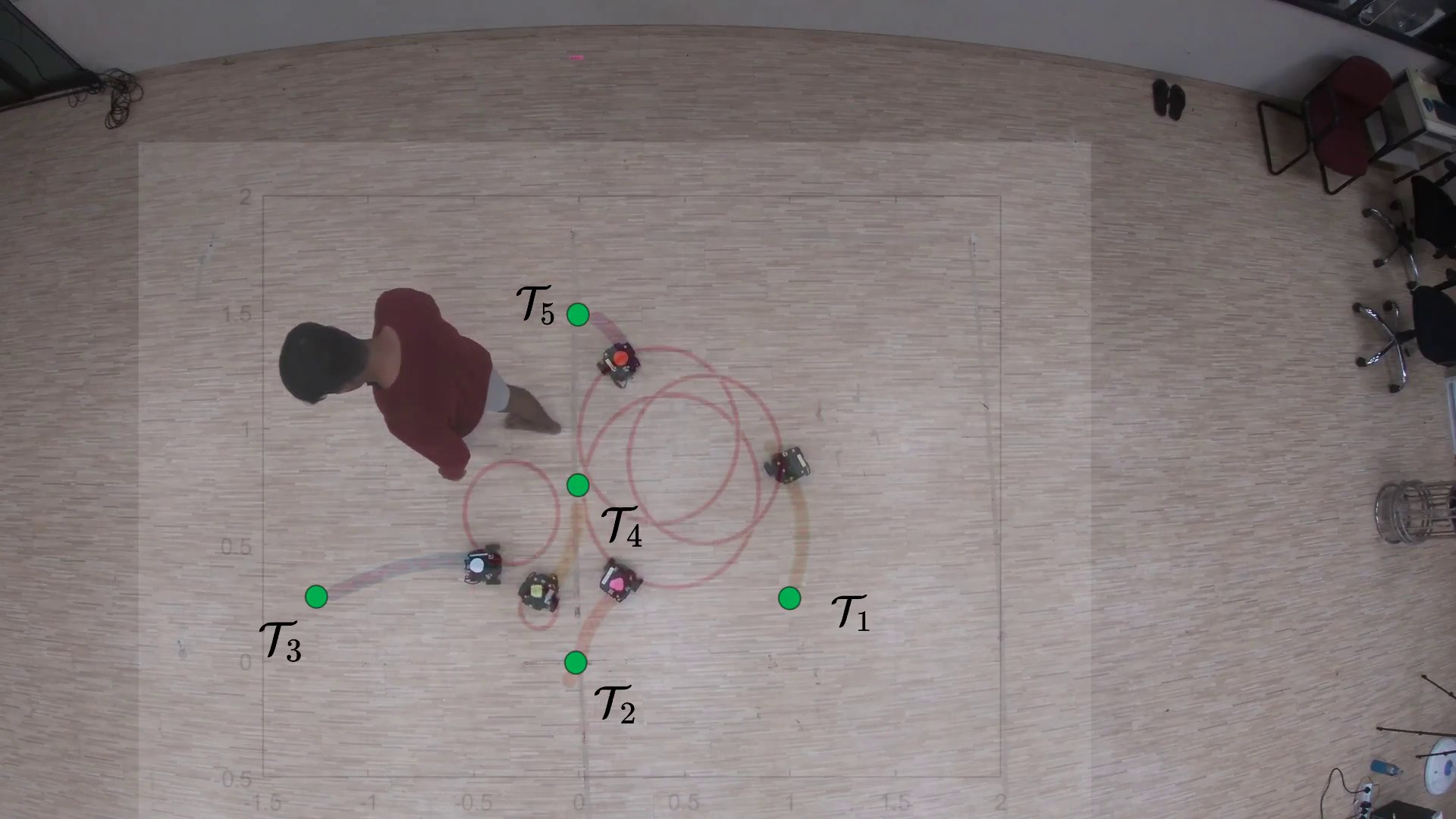}
    \caption{\textbf{T = 10 seconds}. ${L}_i(\mathcal{F}_i^k)$ of all agents grow towards their respective $\mathcal{T}_i$ as the human moves away from the workspace.}
    \label{fig:00_10}
    \end{subfigure}
    \begin{subfigure}[b]{0.66\columnwidth}
    \includegraphics[width = \columnwidth]{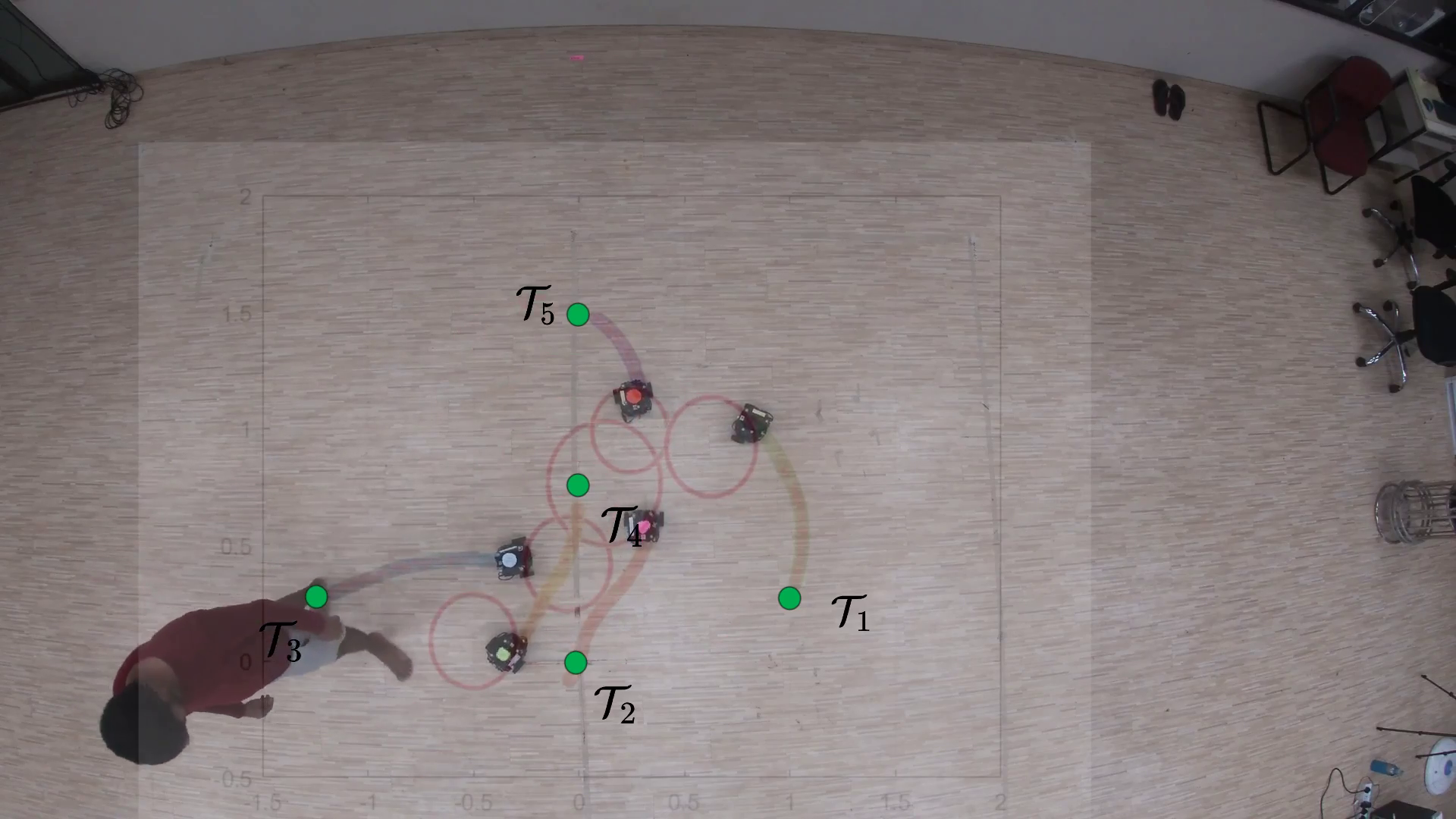}
    \caption{\textbf{T = 13 seconds}. ${L}_i(\mathcal{F}_i^k)$ of $\mathcal{R}_i$ (bright green marker) shrinks (collision avoidance) while other agents move toward their $\mathcal{T}_i$.}
    \label{fig:00_13}
    \end{subfigure}
    \begin{subfigure}[b]{0.66\columnwidth}
    \includegraphics[width = \columnwidth]{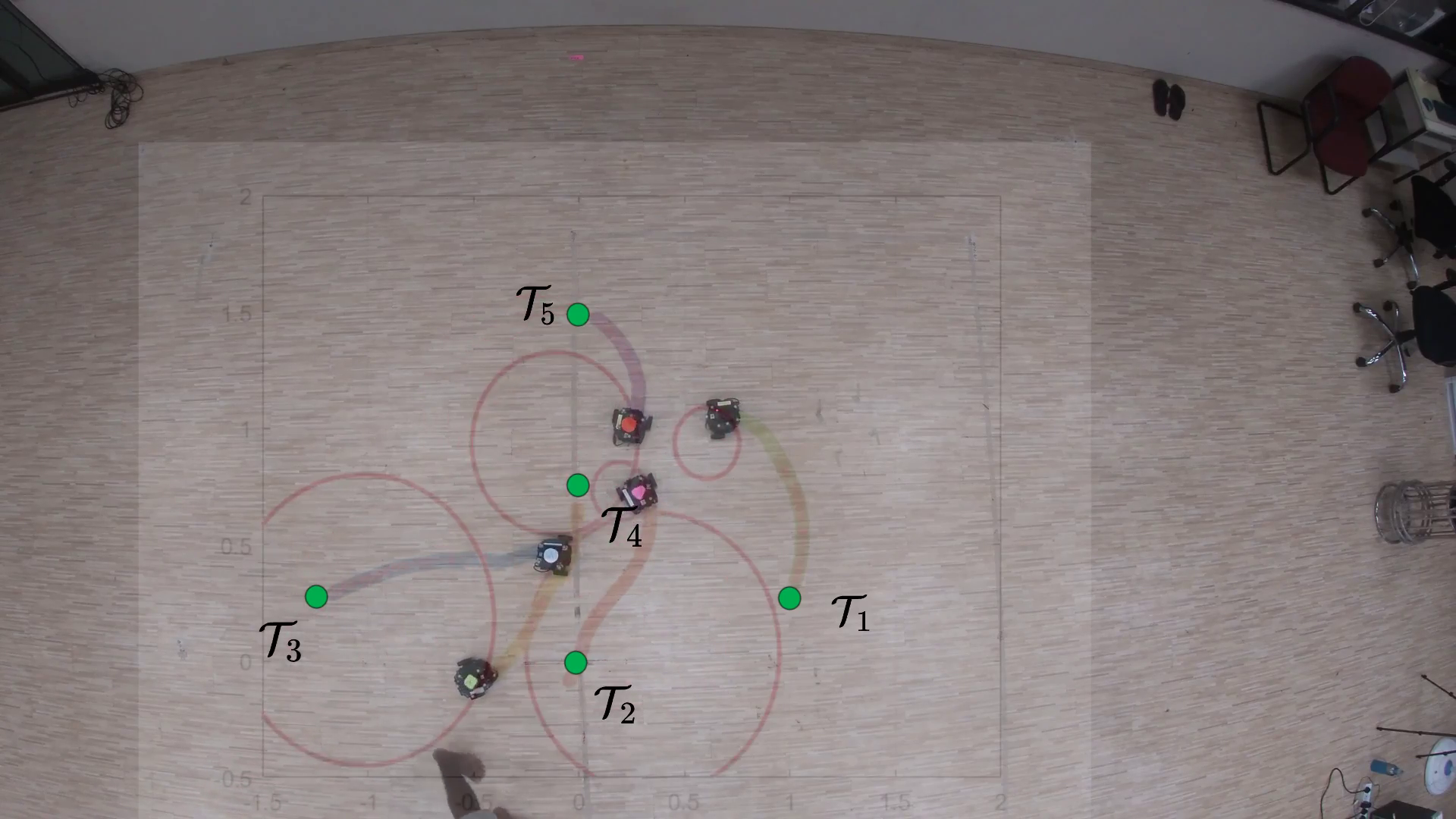}
    \caption{\textbf{T = 16 seconds}. As the human leaves, the ${L}_i(\mathcal{F}_i^k)$ for every $\mathcal{R}_i$ approach their respective $\mathcal{T}_i$ within the sensed bounds of safety.}
    \label{fig:00_16}
    \end{subfigure}
    \begin{subfigure}[b]{0.66\columnwidth}
    \includegraphics[width = \columnwidth]{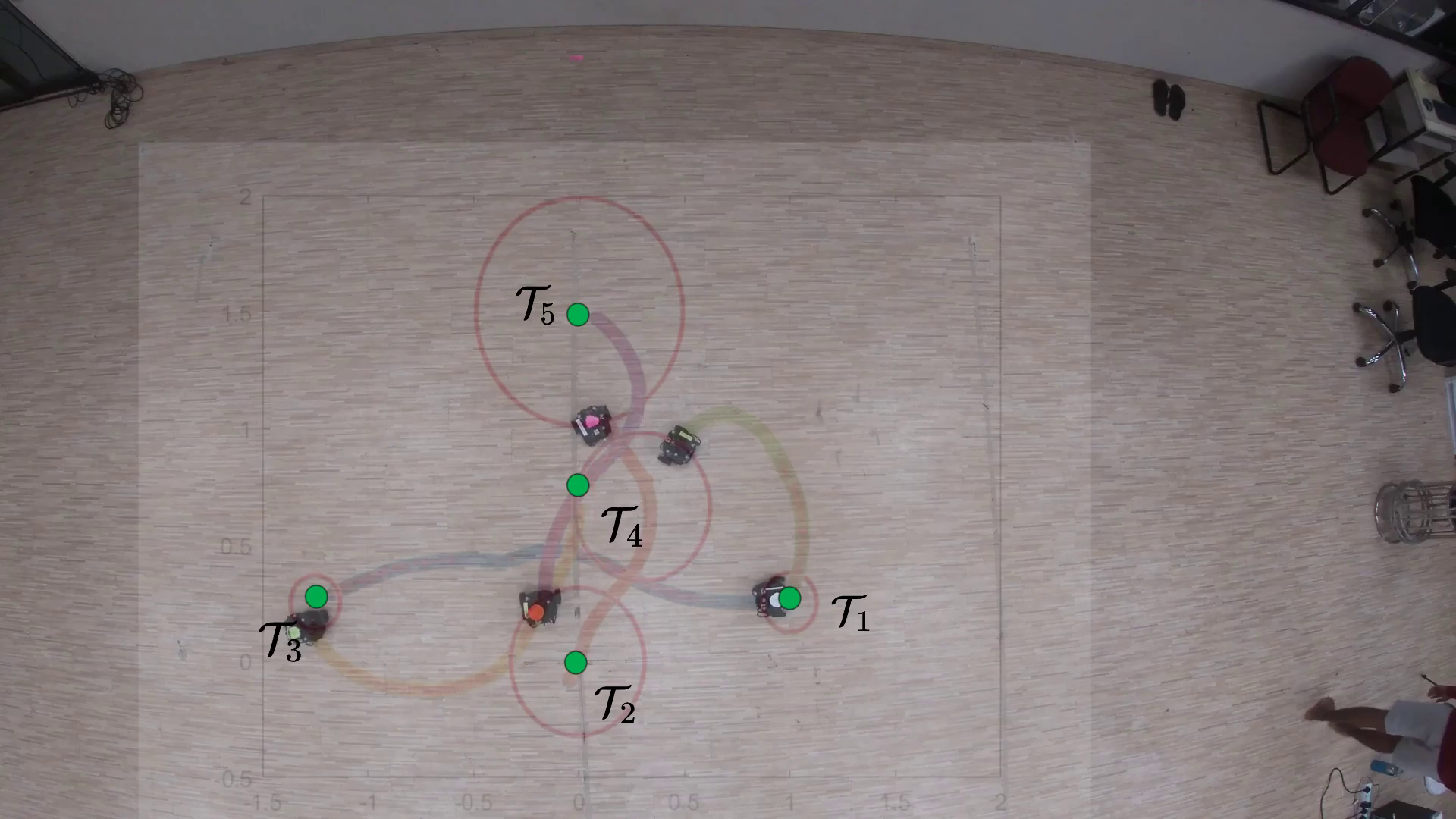}
    \caption{\textbf{T = 27 seconds.} ${L}_i(\mathcal{F}_i^k)$ have latched onto the $\mathcal{T}_i$ for 4 agents (pink, orange are approaching and bright green, white complete)}
    \label{fig:00_27}
    \end{subfigure}
    \begin{subfigure}[b]{0.66\columnwidth}
    \includegraphics[width = \columnwidth]{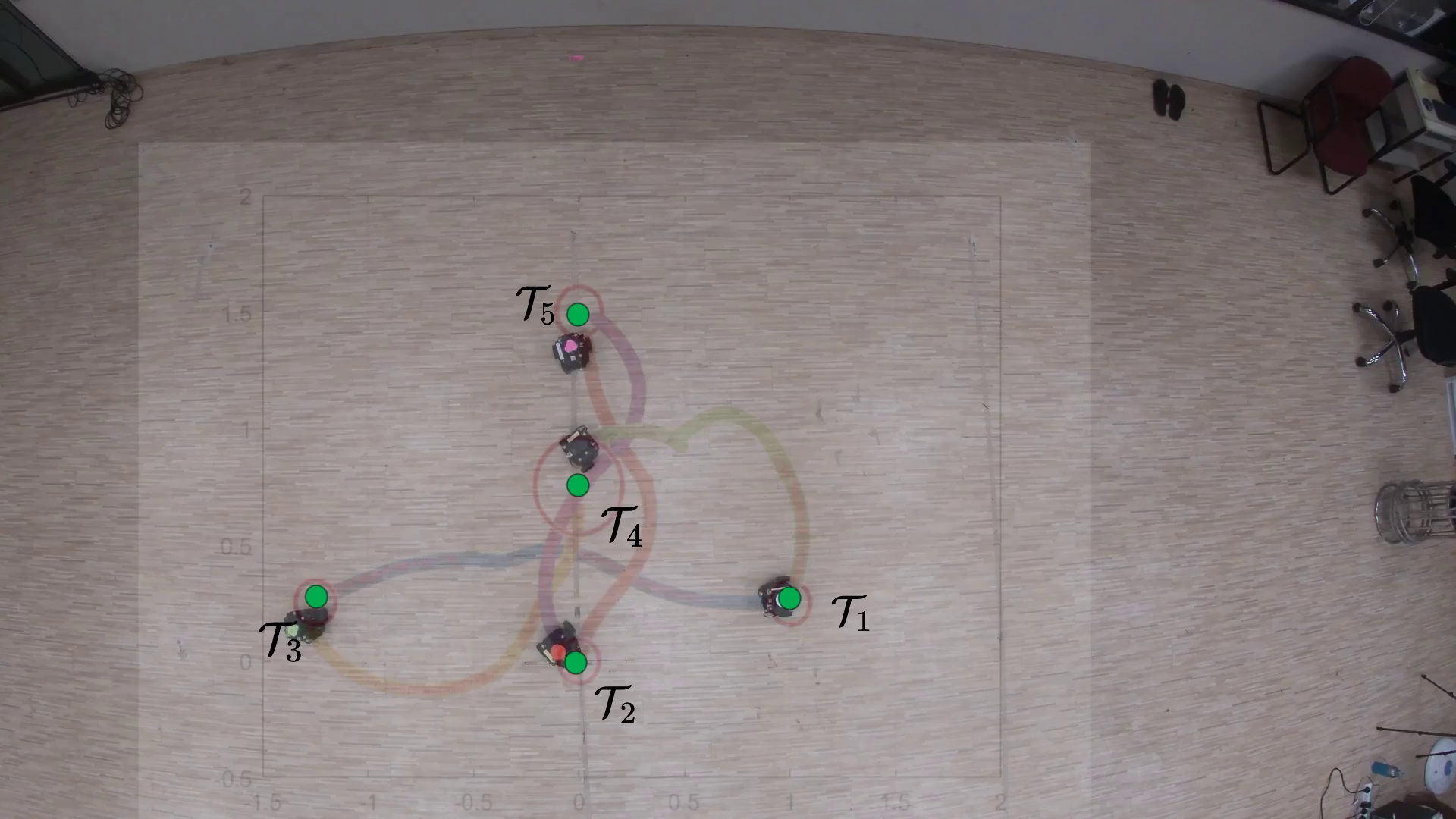}
    \caption{\textbf{T = 33 seconds}. All the agents have accomplished their navigation objective. The line traces represent the overall path induced.}
    \label{fig:00_33}
    \end{subfigure}
    \caption{Experiments to demonstrate the self-navigator deployed on Turtlebot3 robots}
    \label{fig:Turtlebot3Experiments}
\end{figure*}

The proposed self-navigation algorithm for multi-agent navigation is tested and validated in a human-influenced environment consisting of multiple Turtlebot3 agentic platforms (Burger variant). Turtlebot3 is a cylindrical structured differential drive robot with a diameter of 105 mm, and runs on a Raspberry Pi 4 Model B on-board computer. It is equipped with a $360^{\circ}$ LDS-02 LiDAR for sensing the bounds of safety.

The schematic representing the experimental setup at ARMS Lab, IIT Bombay, is illustrated in Fig.~\ref{fig:ExperimentSchematic}. It consists of ViCON motion capture system and 5 Turtlebot3 (with infrared (IR) markers) agents placed at arbitrary initial positions in space. The visual data obtained through the IR cameras is further processed by the motion capture system, and the ground truth of each robot's pose is broadcasted wirelessly through a computer connected to the local network. The navigation targets are specified in the global frame of the motion capture system, and navigation objectives are set to autonomously reshuffle their positions $(\{\mathcal{T}_i\})$ amongst themselves. Since the navigation objectives are specified in the global frame, the agents subscribe to the motion capture system to localize themselves. The control inputs for safe navigation are computed based on local information of the operating environment sensed through the onboard LiDAR. 

To test the online reactive capabilities of the proposed algorithm in responding to uncertainties encountered during run-time, a scenario is created such that a human intervenes in the operating scene and its study is presented in Fig. \ref{fig:Turtlebot3Experiments}. The boundary of $L_i(\mathcal{F}^k_i)$ at the $k^{\text{th}}$ planning instance associated with $\mathcal{R}_i$ is represented to the red circles. During navigation, it could be observed $L_i(\mathcal{F}^k_i)$ is selected such that they reside within the bounds of safety (dynamically varying) identified based on the sensed information during run-time. Eventually, the agents safely converge to their respective $\mathcal{T}_i$. On average, the onboard (agent $\mathcal{R}_i$) computational time recorded on Raspberry Pi 4 for determining invariant sets for control computations was approximately 57 milliseconds at every planning instance.

\subsection{Simulation Results}
\label{sec:4A}
\begin{figure*}
    \centering
    \includegraphics[width = 2\columnwidth]{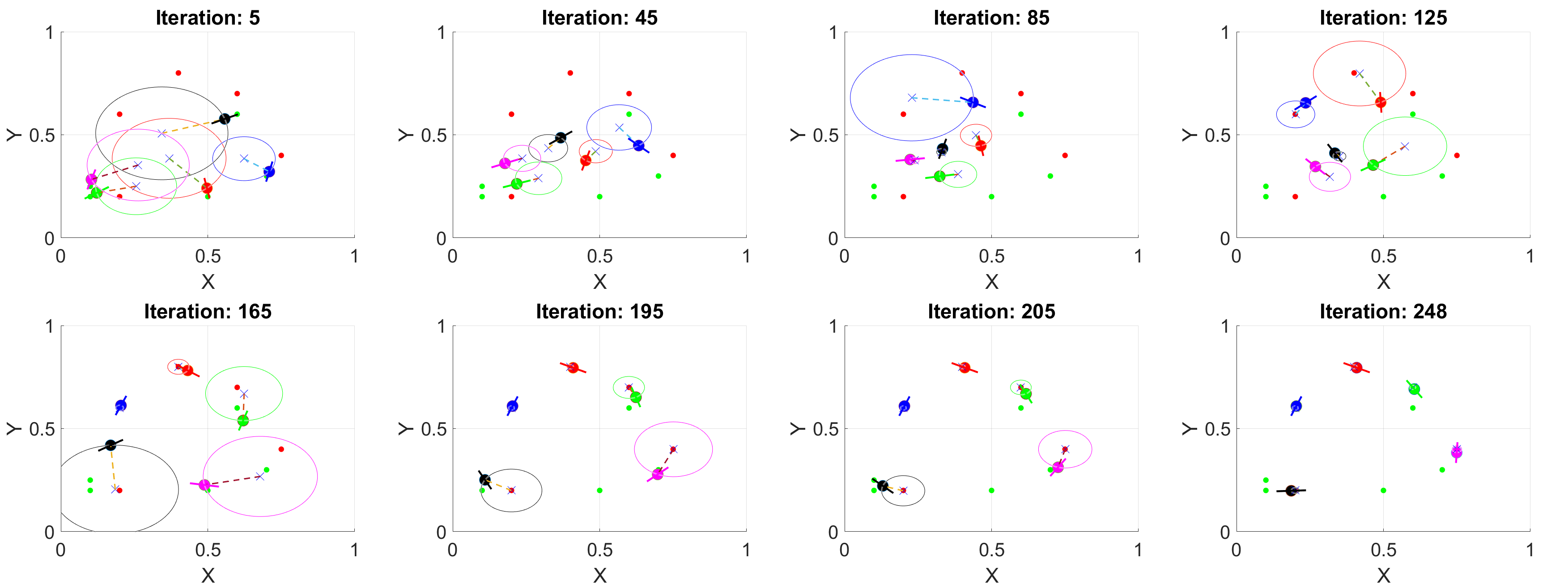}
    \caption{Stages of multi-agent navigation in a chronological sequence (row-wise, left-to-right) of planning instances.}
    \label{fig:Nav_snapshots_merged}
\end{figure*}

\begin{figure*}
    \centering
    \begin{subfigure}[b]{0.5\columnwidth}
    \includegraphics[width = \columnwidth]{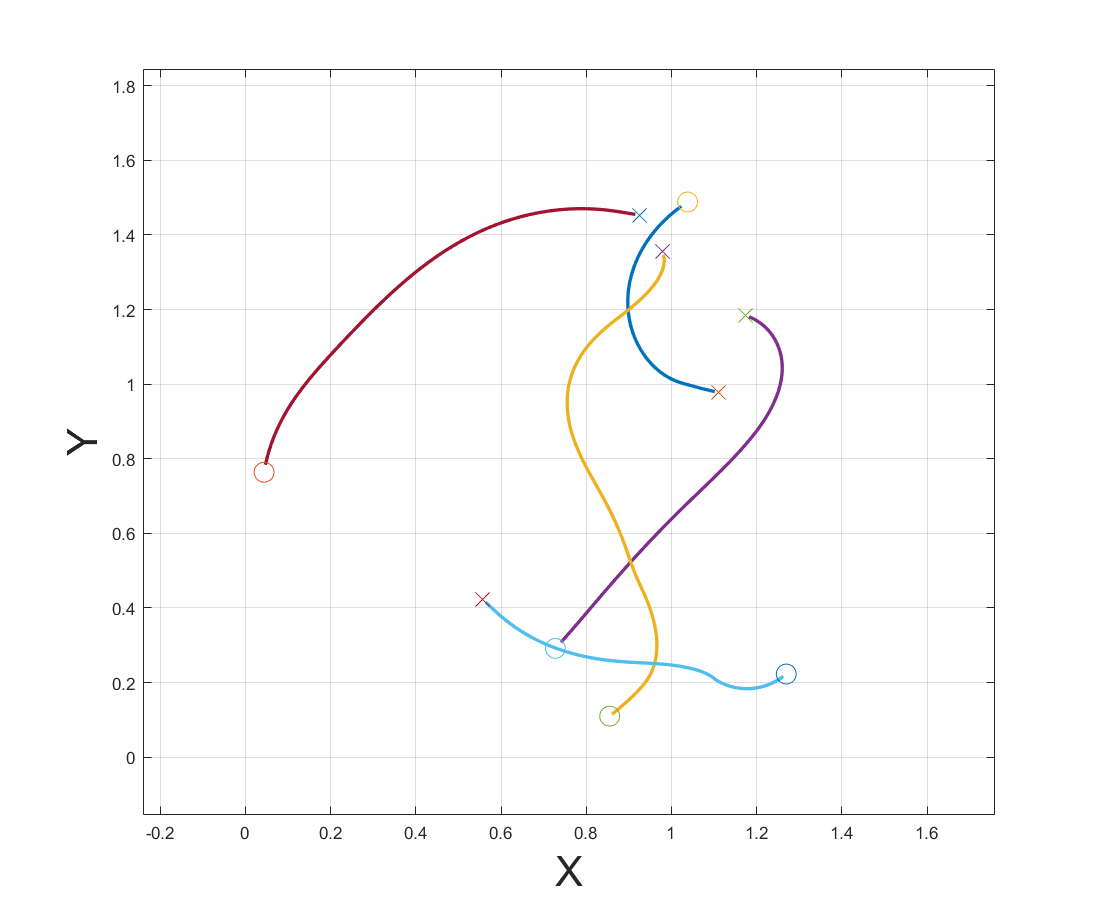}
    \caption{\textbf{Scenario 1: 5 agents}.}
    \label{fig:2D_1}
    \end{subfigure}
        \begin{subfigure}[b]{0.5\columnwidth}
    \includegraphics[width = \columnwidth]{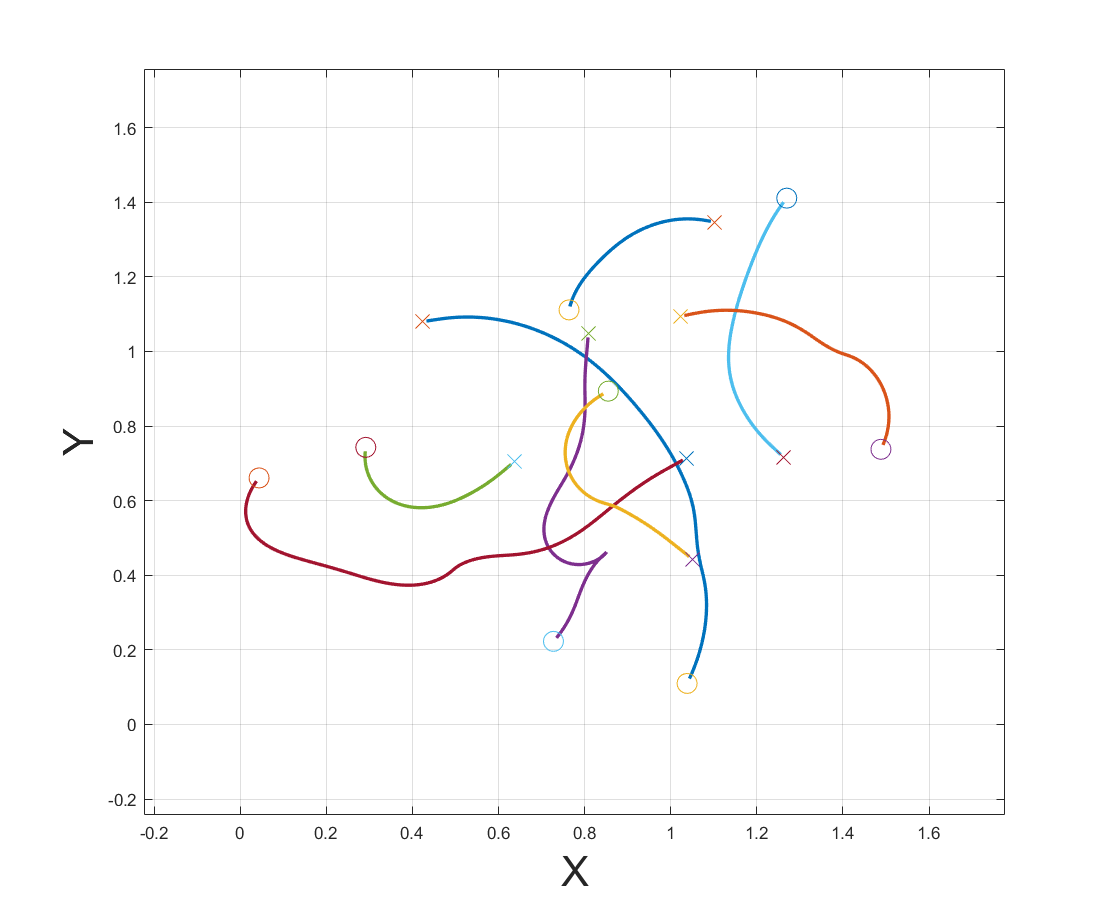}
    \caption{\textbf{Scenario 2: 8 agents}.}
    \label{fig:2D_2}
    \end{subfigure}
        \begin{subfigure}[b]{0.5\columnwidth}
    \includegraphics[width = \columnwidth]{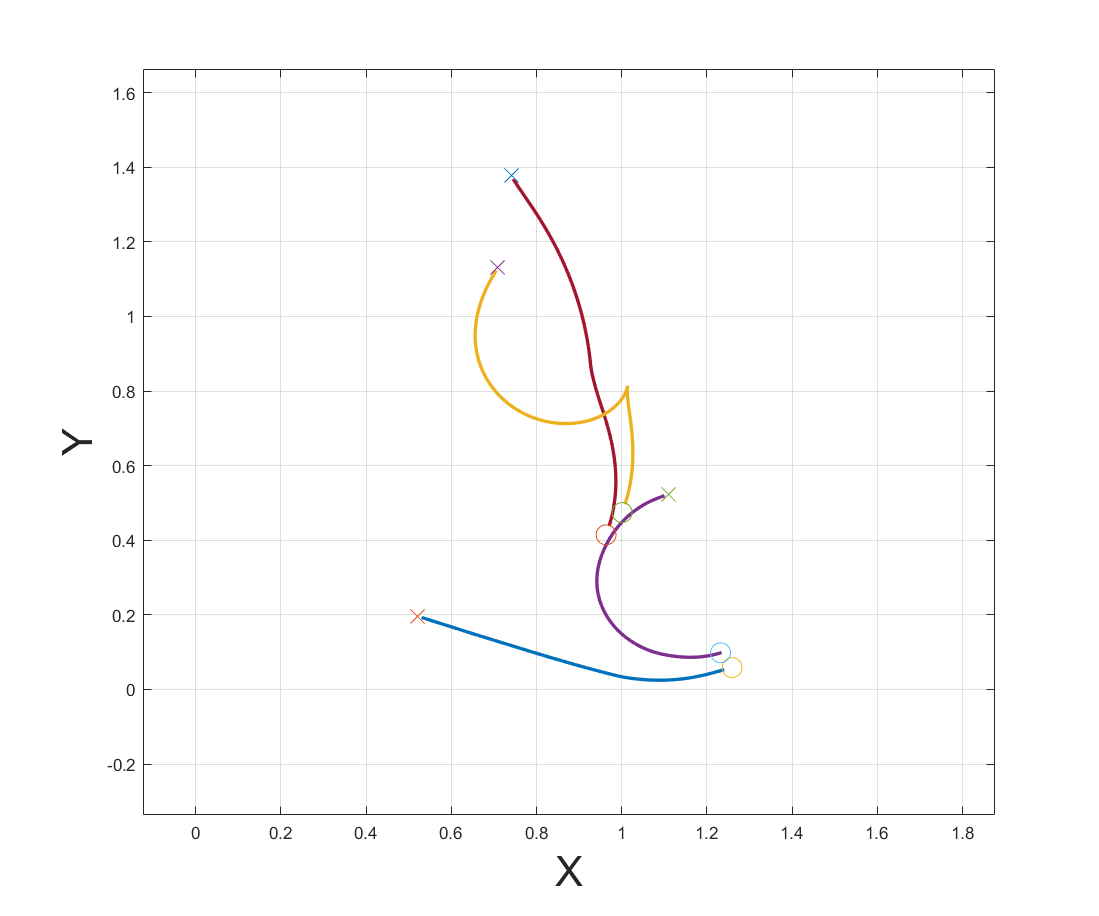}
    \caption{\textbf{Scenario 3: 4 agents}.}
    \label{fig:2D_3}
    \end{subfigure}
    \begin{subfigure}[b]{0.5\columnwidth}
    \includegraphics[width = \columnwidth]{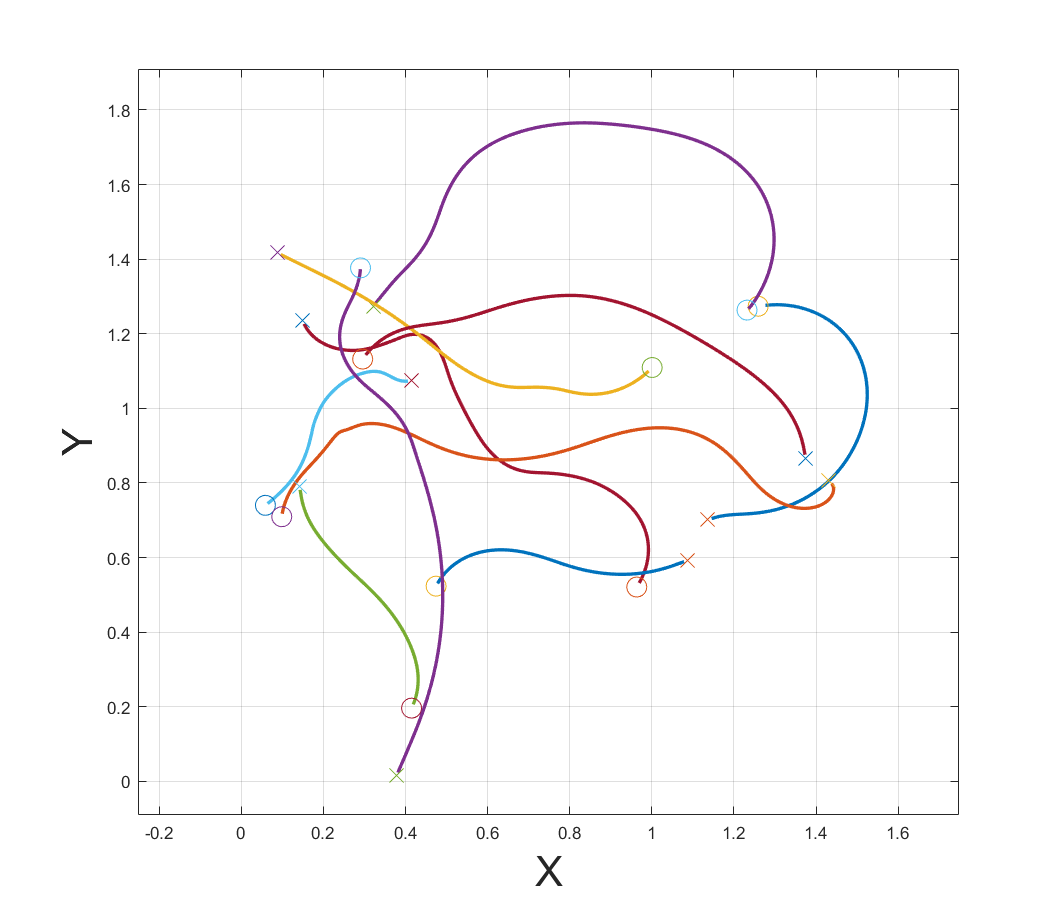}
    \caption{\textbf{Scenario 4: 10 agents}.}
    \label{fig:2D_4}
    \end{subfigure}
    \begin{subfigure}[b]{0.5\columnwidth}
    \includegraphics[width = \columnwidth]{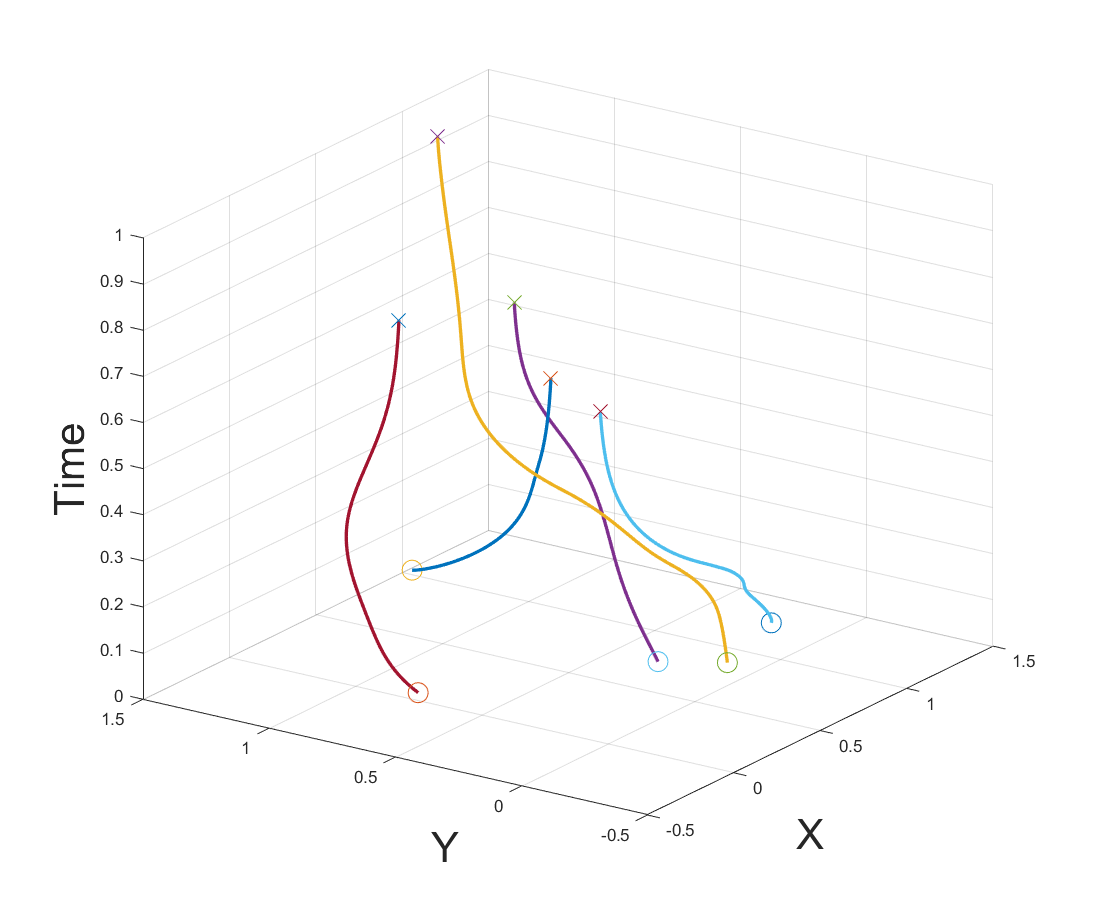}
    \caption{\textbf{Scenario 1: 5 agents}.}
    \label{fig:3D_1}
    \end{subfigure}
    \begin{subfigure}[b]{0.5\columnwidth}
    \includegraphics[width = \columnwidth]{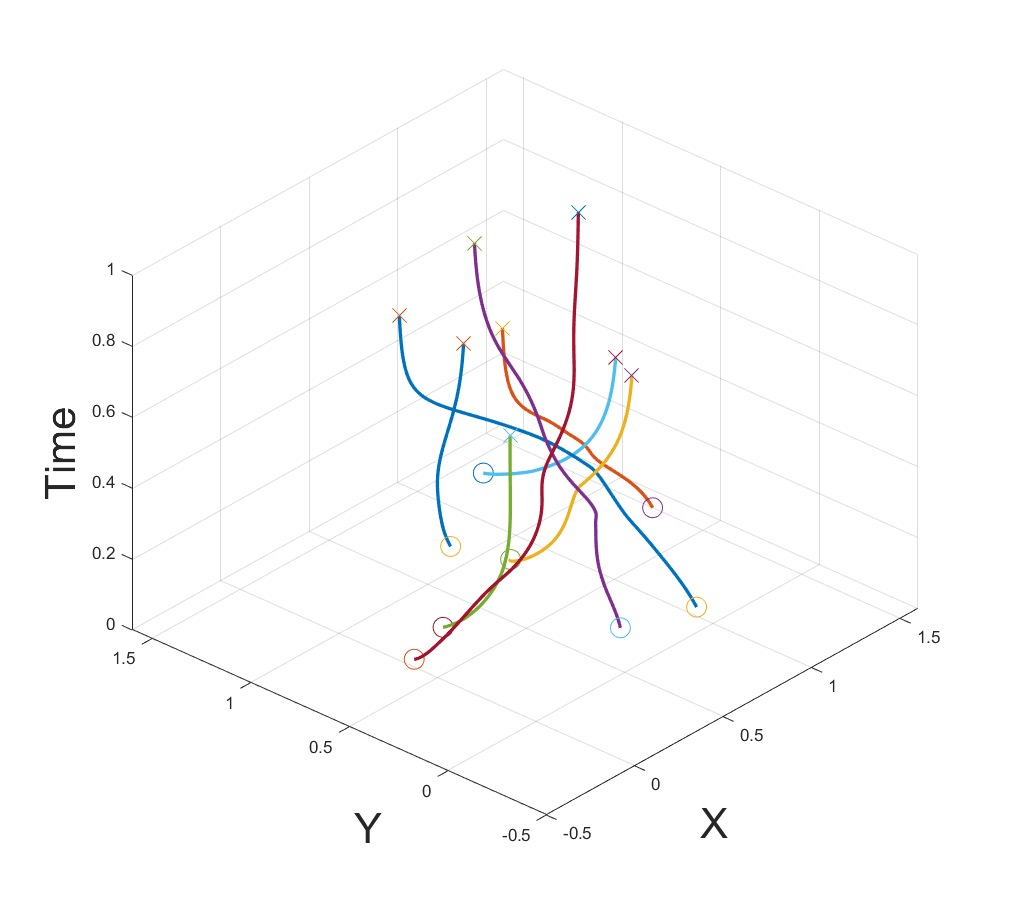}
    \caption{\textbf{Scenario 2: 8 agents}.}
    \label{fig:3D_2}
    \end{subfigure}
    \begin{subfigure}[b]{0.5\columnwidth}
    \includegraphics[width = \columnwidth]{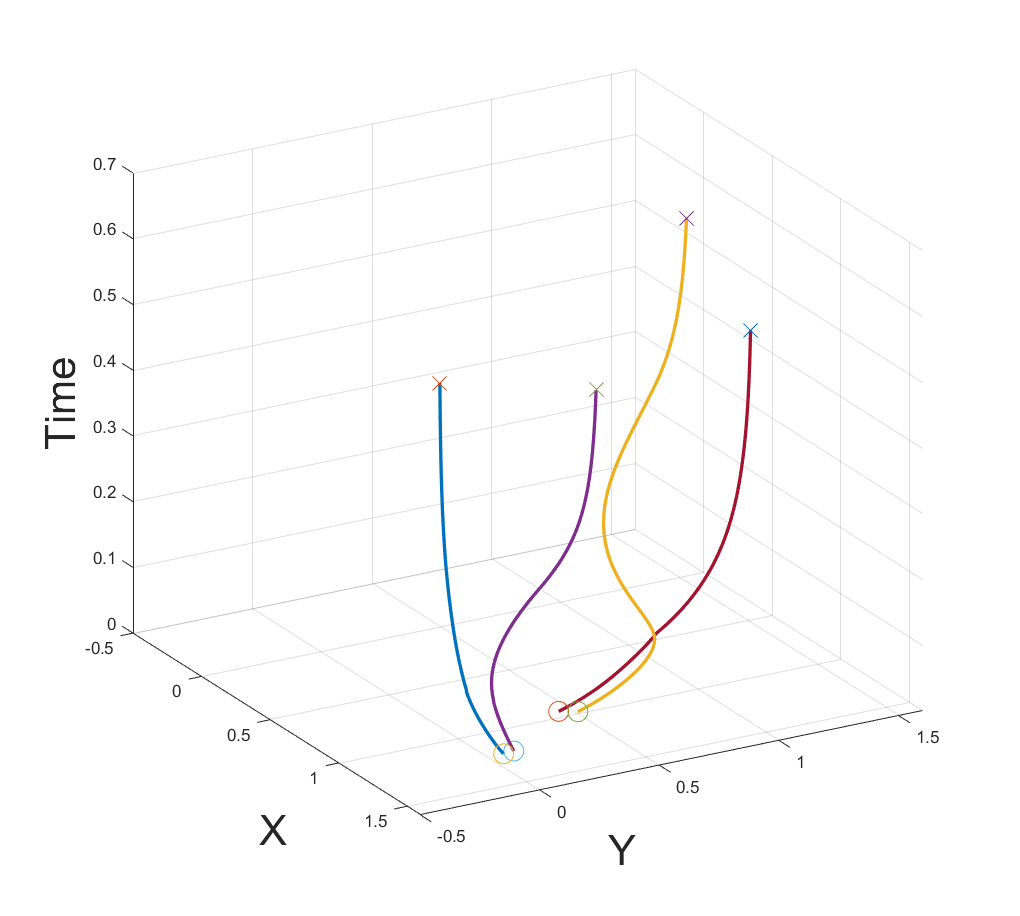}
    \caption{\textbf{Scenario 3: 4 agents}.}
    \label{fig:3D_3}
    \end{subfigure}
    \begin{subfigure}[b]{0.5\columnwidth}
    \includegraphics[width = \columnwidth]{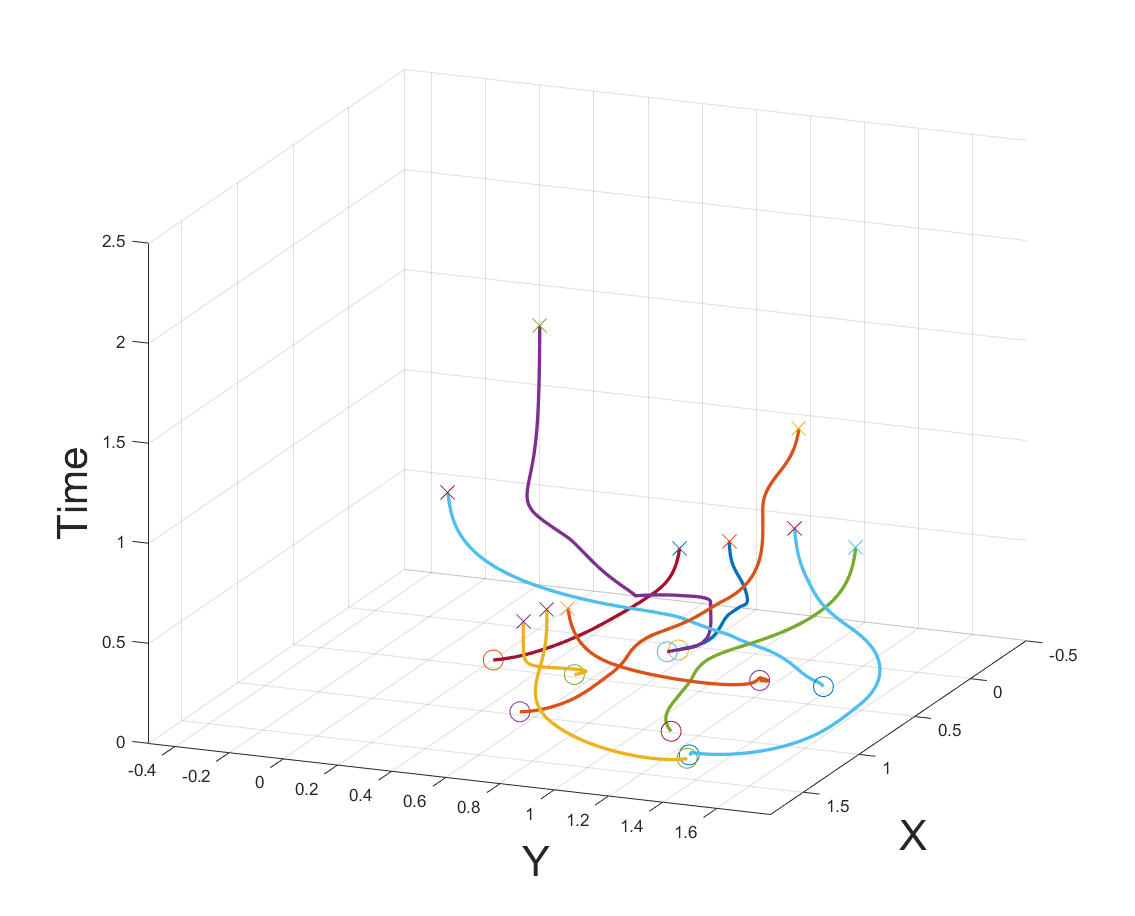}
    \caption{\textbf{Scenario 4: 10 agents}.}
    \label{fig:3D_4}
    \end{subfigure}
    \caption{Example multi-agent navigation scenarios with varying number of agents. The plots in vertical alignment represent the same scenario with the plots in the bottom row representing the collision-free nature of the naturally induced trajectories.}
    \label{fig:MultipleSimRuns}
\end{figure*}

The self-navigation algorithm on each $\mathcal{R}_i$ is tested through simulations over multiple test cases. Each test case consists of varying numbers of agents with each $\mathcal{R}_i$ assigned an arbitrary target $\mathcal{T}_i$ for navigation. In the simulations, the angular space in the ego-centric frame of $\mathcal{R}_i$ has 64 angular discretizations at which the range measurements are obtained.

In Fig,~\ref{fig:Nav_snapshots_merged}, different stages of an example navigation scenario involving multiple non-holonomic agents (agents) are illustrated. The colored dots with the line through them indicate the position of the agents with their instantaneous orientation. 
The green dots represent the initial points from which the agents begin their navigation and the red dots are the corresponding target points of navigation $\mathcal{T}_i$. The boundary of the invariant set associated with the onboard 
feedback controller $\mathcal{F}_i(R^k_i,\psi^k_i;\mathcal{W}^k_i)$ of each of the agents is represented through the corresponding colored circles. The identified target point $\mathcal{W}^k_i$ based on which the feedback control inputs are computed, is represented through blue $\times$. From the simulation results, it could be observed that during each planning instance, the proposed algorithm identifies invariant sets such that they do not violate the bounds of safety based on the measurements obtained during run-time for any given $\mathcal{R}_i$. Eventually, $\mathcal{W}^k_i$ approaches and converges onto their respective $\mathcal{T}_i$ using the proposed planning strategy.

In Fig. $\ref{fig:MultipleSimRuns}$, the proposed algorithm is tested on multiple scenarios with varying numbers of agents and initial conditions. The plots on the same vertical alignment represent similar operating scenarios. The top row consists of the path induced by the control inputs on the 2D operating workspace generated through the proposed algorithm. The plots on the bottom row represent the collision-free position trajectories induced by deploying the control inputs generated through the proposed algorithm. In each of the plots, the symbols `$\circ$' and `$\times$' represent the starting and the target points for each agents.

\subsection{Comparisons}
\label{sec:4B}
In this section, a comparison is presented in Table~\ref{table:Comparisons} between the proposed navigation algorithm, and a decoupled strategy where a straight-line path is pre-planned with the agents equipped with an onboard reactive navigation algorithm to avoid collisions while attempting to track the same.

Since curvature is a measure of the smoothness of the induced paths, the proposed algorithm clearly outperforms decoupled strategies in this aspect as the average path curvatures recorded on deploying the proposed algorithm are consistently much smaller when compared to its counterpart. The larger path curvatures induced is due to the requirement of strict adherence to pre-planned paths. The resulting path lengths traversed by each agent are also comparable in the scenarios under comparison. The slightly increased path lengths in the case of the proposed algorithm are due to the resulting smoothness of turns executed by the agents during navigation. 

\begin{table*}[t]
\caption{Comparisons - Proposed Algorithm vs Decoupled Strategy (Random initial conditions)}
\label{table:Comparisons}
\begin{center}
\begin{tabular}{|m{5em}| |m{9em}| |m{8em}| |m{5em}| |m{9em}| |m{8em}|} 
 \hline
 \textbf{Number of Agents} & \textbf{Path Length Traversed (Decoupled)} & \textbf{Path Length Traversed (Proposed)} & \textbf{Path length difference} & \textbf{Average Path Curvature (Decoupled)} & \textbf{Average Path Curvature (Proposed)} \\ [0.4ex]
  \hline  
  \multirow{4}{*}{$4$}
  & $3.19$  &$2.79$ &-0.4 & $(2.31\times 10^2)$ &$2.92$\\
  \cline{2-6}
 & $2.47$  &$2.66$ &+0.19 & $(2.95\times 10^2)$ &$5.08$\\
 \cline{2-6}
 & $3.56$  &$3.64$ &+0.08 & $(2.71\times 10^2)$ &$11.9$\\
 \cline{2-6}
 & $3.15$  &$3.06$ &-0.09 & $(2.56\times 10^2)$ &$3.04$\\
 \hline
 \hline
 \multirow{4}{*}{$4$}
  & $3.58$  &$3.76$ &-0.18 & $(2.18\times 10^2)$ &$4.69$\\
 \cline{2-6}
 & $1.38$  &$1.27$ &-0.09 & $(2.23\times 10^2)$ &$3.11$\\
 \cline{2-6}
 & $3.58$  &$3.68$ &+0.1 & $(2.38\times 10^2)$ &$4.54$\\
 \cline{2-6}
 & $2.33$  &$2.39$ &+0.06 & $(2.78\times 10^2)$ &$4.03$\\
 \cline{2-6}
 \hline
 \hline
 \multirow{4}{*}{$4$}
   & $0.69$  &$0.702$ &+0.012 & $(0.34\times 10^2)$ &$5.29$\\
 \cline{2-6}
   & $2.61$  &$2.97$ &+0.36 & $(2.88\times 10^2)$ &$3.83$\\
 \cline{2-6}
   & $5.04$  &$3.2$ &-1.8 & $(2.79\times 10^2)$ &$6.93$\\
 \cline{2-6}
   & $1.44$  &$1.4$ &-0.04 & $(2.86\times 10^2)$ &$5.94$\\
 \cline{2-6}
 \hline
 \hline
\end{tabular}
\end{center}
\end{table*}

\subsection{Discussion - Measurement uncertainties}
\label{sec:4D}
In practice, the onboard LiDAR measurements are often noisy. The quality of measured data obtained during run-time depends on environmental factors such as the reflectivity, textures of the reflecting surfaces, interference from other LiDARs (multi-agent scenarios), and other light sources. Furthermore, determining the precise velocities of the neighboring agents is even more arduous through on-board sensing. In general, a lot of computational overhead is required to determine the velocities. Therefore, to truly ensure safety during operation, the onboard motion planning algorithms should be capable of dealing with these uncertainties. In this context, the following modifications are proposed for $\mathcal{C}_m$ based on which $\mathcal{D}^k_i$ is constructed in section \ref{sec:3A}. Assuming $\mathcal{R}_i$ is aware of the velocity limitations (say, $|\vec{v}_m| \leq v_{\max}$) of moving agents in its vicinity $(\vec{v}_m = |\vec{v}_m|\hat{v}_m)$,
\begin{itemize}
    \item When there are significant uncertainties in both the measured direction as well as the magnitude of $\vec{v}_m$.
\begin{align}
    d_{\min} (P_m,L^k_i) > \frac{v_{\max}}{f_p}
    \label{eqn:mod_vel_consts_1}
\end{align}
    \item When there are significant uncertainties in the measured direction $\hat{v}_m$, but the magnitude $|\vec{v}_m|$ is available to a good degree of accuracy.
    \begin{align}
        d_{\min} (P_m,L^k_i) > \frac{|\vec{v}_{m}|}{f_p}
        \label{eqn:mod_vel_consts_2}
    \end{align}
    \item When there are significant uncertainties in the magnitude $|\vec{v}_m|$, but its direction $\hat{v}_m$ is available to a good degree of accuracy.
    \begin{align}
        d_{\hat{v}_m} (P_m,L^k_i) > \frac{v_{\max}}{f_p}
        \label{eqn:mod_vel_consts_3}
    \end{align}
\end{itemize}

The equations \eqref{eqn:mod_vel_consts_1}, \eqref{eqn:mod_vel_consts_2}, \eqref{eqn:mod_vel_consts_3} are independent of the quantities with significant uncertainties, and are replaced with the conditions corresponding to their worst case scenarios. Consequently, they result in conservative bounds for selecting $\mathcal{W}^k_i$ in comparison to the proposed $\mathcal{C}_m$ (as in section $\ref{sec:3B}$) but without compromising on safety at each planning instance.
\section{EXPERIMENTS ON PARALLELIZATION}
\label{sec:5}
\begin{figure*}
    \centering
    \begin{subfigure}[b]{\columnwidth}
    \includegraphics[width = \columnwidth]{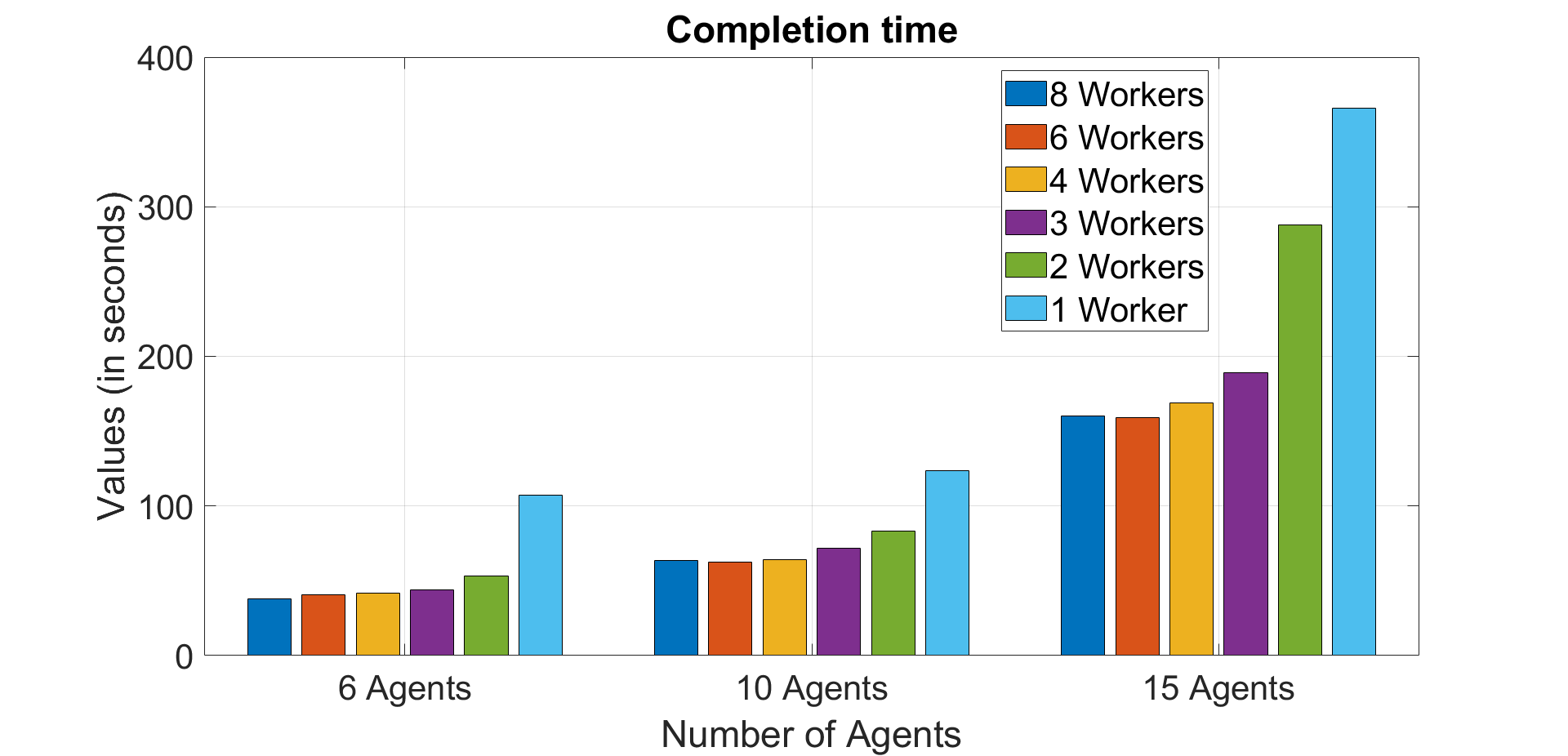}
    \caption{Total completion time with parallelization}
    \label{fig:TotalComputationalTime}
    \end{subfigure}
        \begin{subfigure}[b]{\columnwidth}
    \includegraphics[width = \columnwidth]{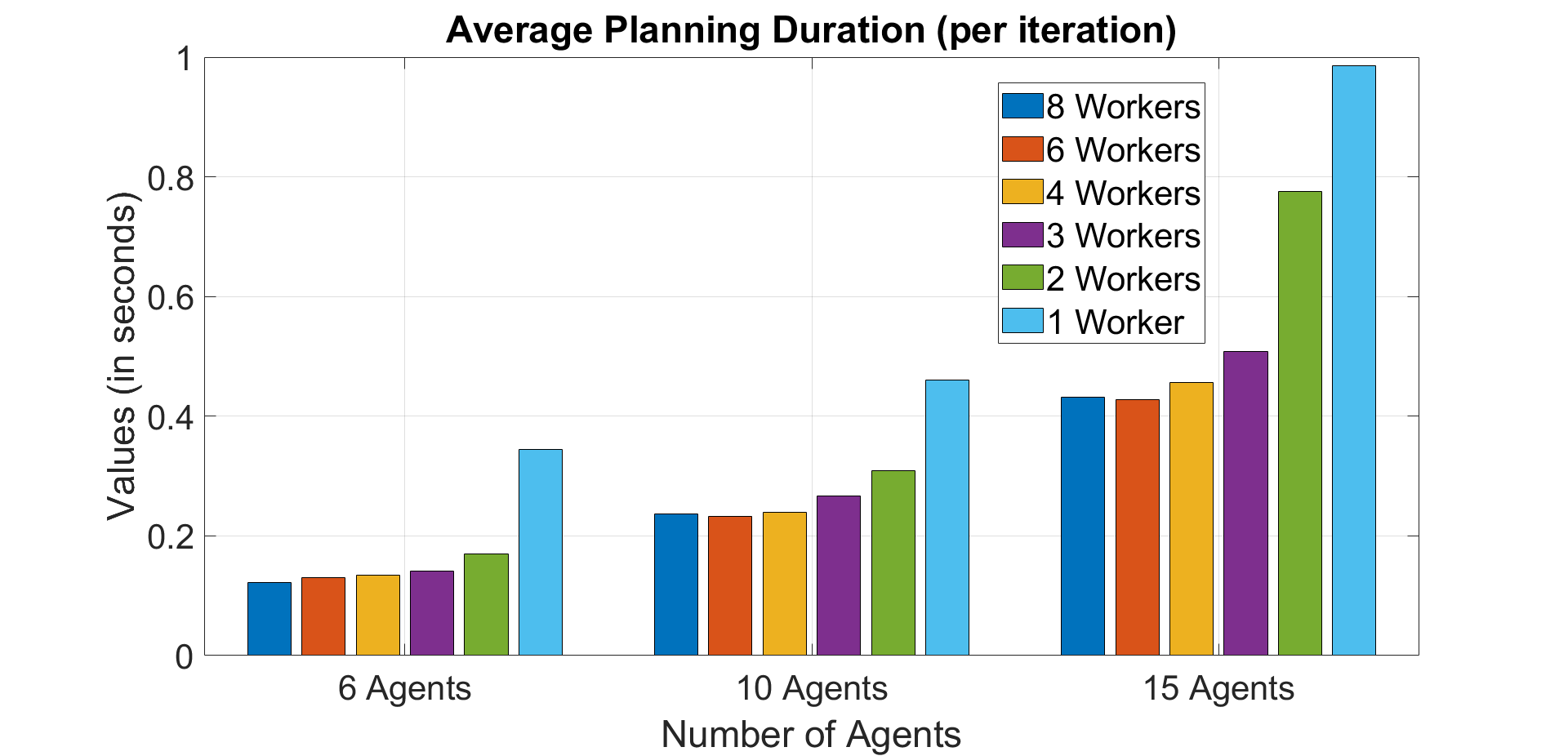}
    \caption{Planning duration per iteration with parallelization}
    \label{fig:AvgPlanperIter}
    \end{subfigure}
        \begin{subfigure}[b]{\columnwidth}
    \includegraphics[width = \columnwidth]{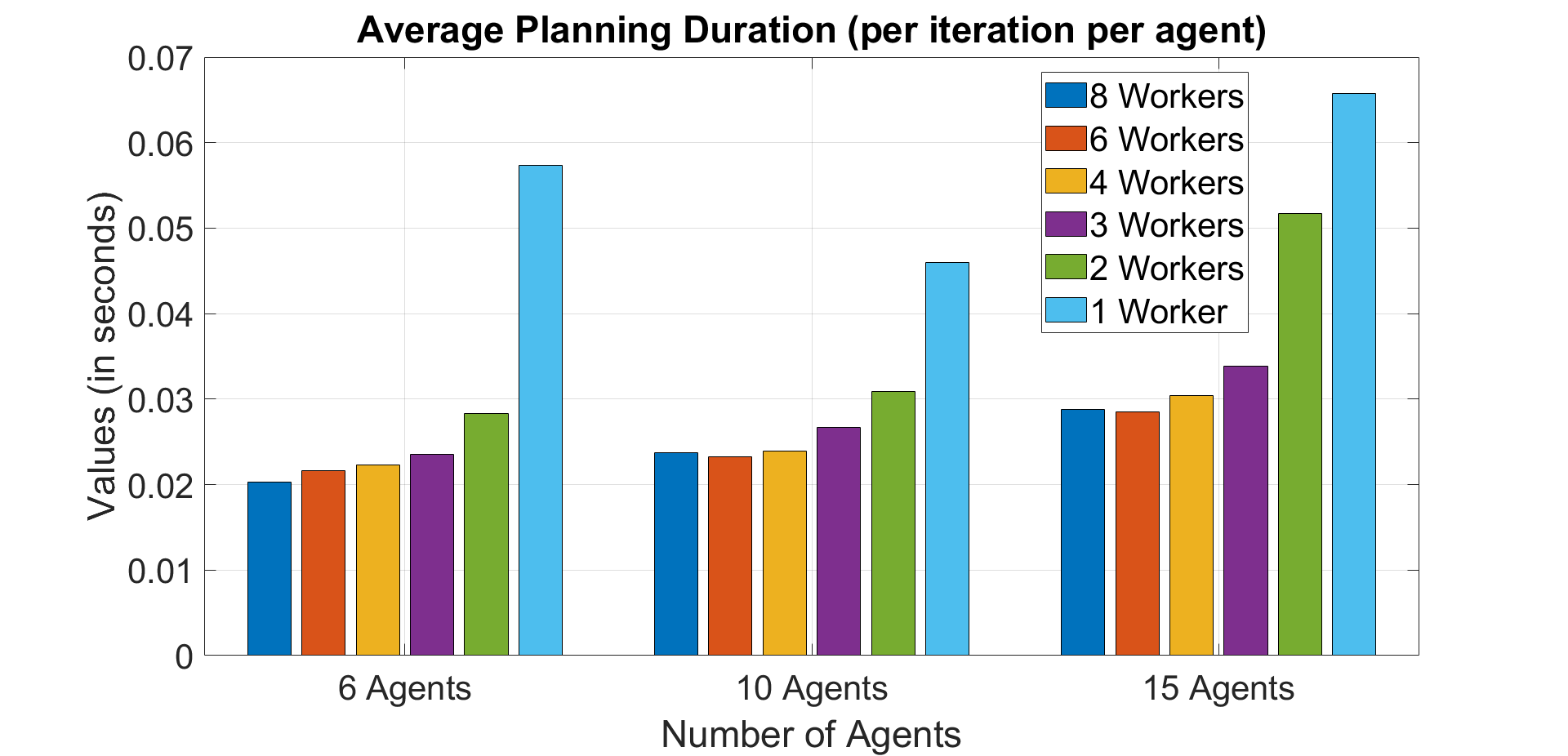}
    \caption{Planning duration per iteration per agent with parallelization}
    \label{fig:AvgPlanperIterperAgent}
    \end{subfigure}
    \begin{subfigure}[b]{\columnwidth}
    \includegraphics[width = \columnwidth]{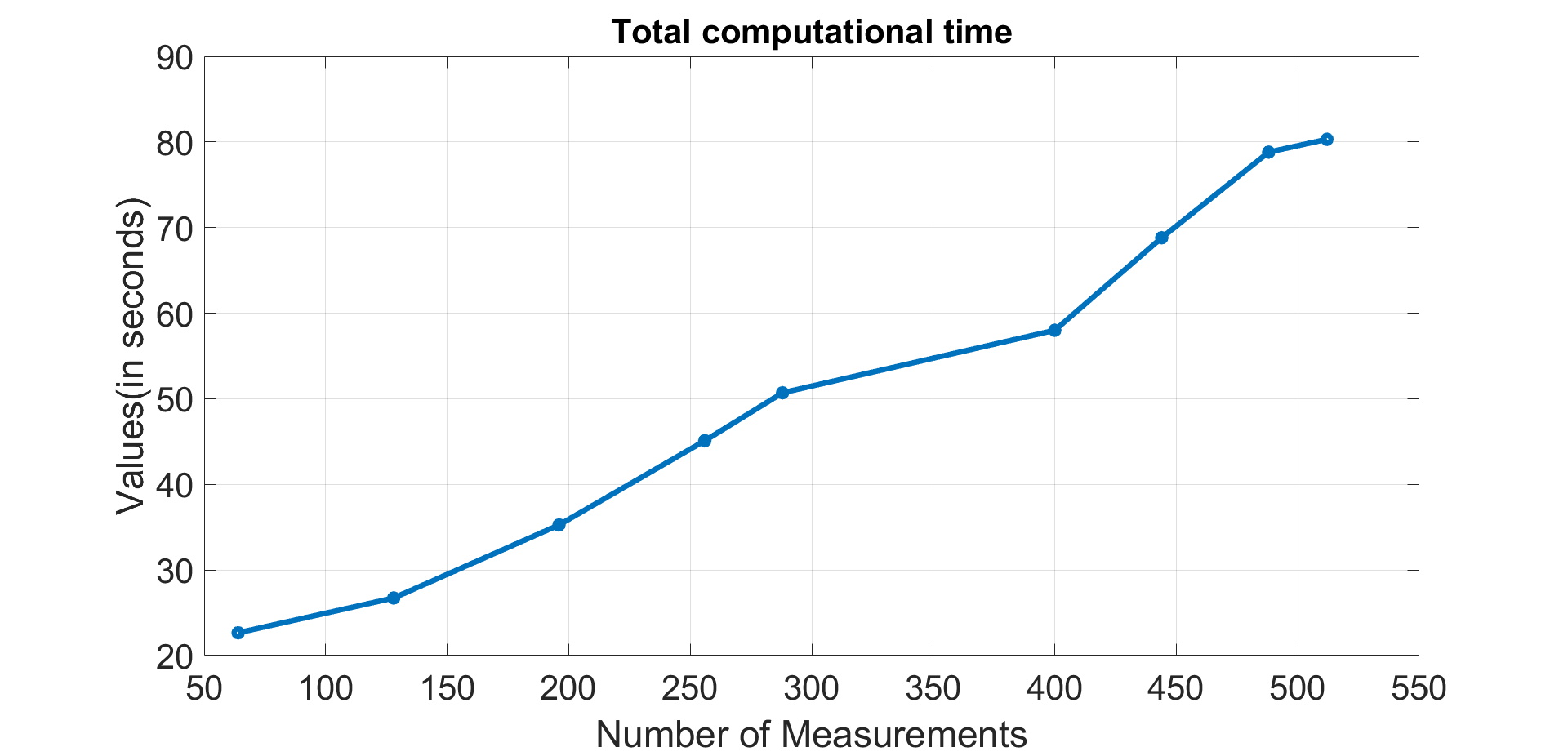}
    \caption{Computational time with an increase in the number of range measurements from the onboard LiDAR (8 workers)}
    \label{fig:MeasurementsCompare}
    \end{subfigure}
    \caption{Computational trends recorded through numerical experiments involving parallel processing.}
\end{figure*}
During each planning instance, the function $\mathcal{D}^k_i$ in the section \ref{sec:3} is constructed online through $\mathcal{M}_i$ based on the sensed information during run-time. A key feature of this construction as in $\eqref{eqn:Opti_Def}$ is that each iteration involved in computing the entries of $\mathcal{D}^k_i$ is independent of each other. Consequently, the computational load involved in constructing $\mathcal{D}^k_i$ can be distributed over multiple computing units operating in parallel. The block diagram for parallel processing is presented in Fig.~\ref{fig:ParallelPlannerFunc} and the associated modification to parallelize the main algorithm \ref{alg:Proposed algorithm} is illustrated through a flowchart in Fig.~\ref{fig:ParallelFlow}. Therefore, given the availability of multiple such computing units on the processor, the proposed planner provides a means to achieve hardware acceleration to hasten online computations.

To quantify the performance of the proposed algorithm, numerical experiments are carried out using the parallel computing toolbox of MATLAB in this section. They are carried out on the HP Zbook Power G8 mobile workstation PC which runs on a 11th Gen i7 processor, 32 GB RAM, and has an inbuilt Nvidia Quadro T1200 graphics processing unit (4 GB DDR6). The proposed algorithm is evaluated through the following metrics:
\begin{itemize}
    \item \textbf{Number of Agents:} Total number of agents in operation for which control commands are computed at each planning instance.
    \item \textbf{Number of workers:} Total number of computing units operating in parallel.
    \item \textbf{Completion Time:} For a random set of initial conditions, this is the total time recorded until convergence is achieved for every agent under consideration. 
    \item \textbf{Average planning duration:} Each iteration consists of computing a set of control commands for every operating agent based on their individual measurements. This metric characterizes the average recorded time of computations.
    \begin{itemize}
        \item \textbf{Average planning duration (per iteration):} This duration is a cumulative measure of computational time elapsed per iteration considering every agent.
        \item \textbf{Average planning duration (per iteration per agent):} This is the average computational time elapsed for each agent within an iteration.
    \end{itemize}
    \item \textbf{Number of measurements:} Total number of simulated LiDAR measurements based on which safe control commands are computed for every agent at each instance.
\end{itemize}

The following trends could be observed from the data obtained through the numerical experiments.

\begin{itemize}
    \item No collisions were observed in any scenario under which the numerical experiments were performed.
    \item An increase in the number of agents resulted in larger overall completion times, and this is a direct reflection of the total number of iterations that elapsed until every agent converged to their respective target points.
    \item An expected trend of decrease in the average planning iteration duration was observed with an increase in the number of workers deployed for online computations A hardware acceleration of up to 3 times lesser computational times is achieved. 
    \item An increase in the number of measurements resulted in increased computational times. Since every measurement is taken into consideration (for construction of $\mathcal{D}_k^i$) for computing safe control inputs - $U(R^k_i,\psi^k_i;\mathcal{W}^k_i)$, this is a natural trend that is expected.
\end{itemize}

\begin{figure}
    \centering
    \includegraphics[width = \columnwidth]{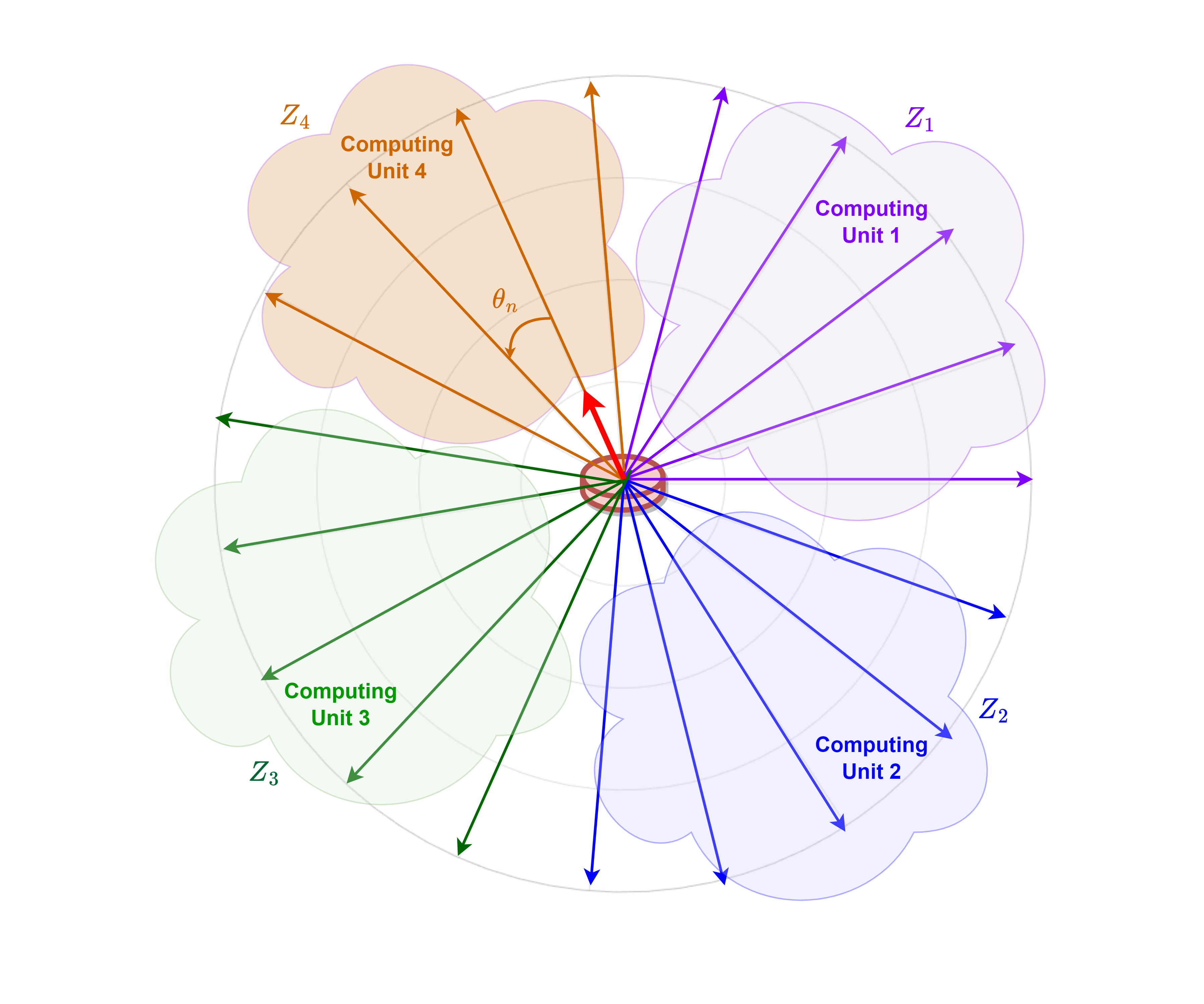}
    \caption{Parallelization of the planner function $\mathcal{D}^k_i$ construction through the distribution of the computational load.}
    \label{fig:ParallelPlannerFunc}
\end{figure}

\begin{figure}
    \centering
    \includegraphics[width = \columnwidth]{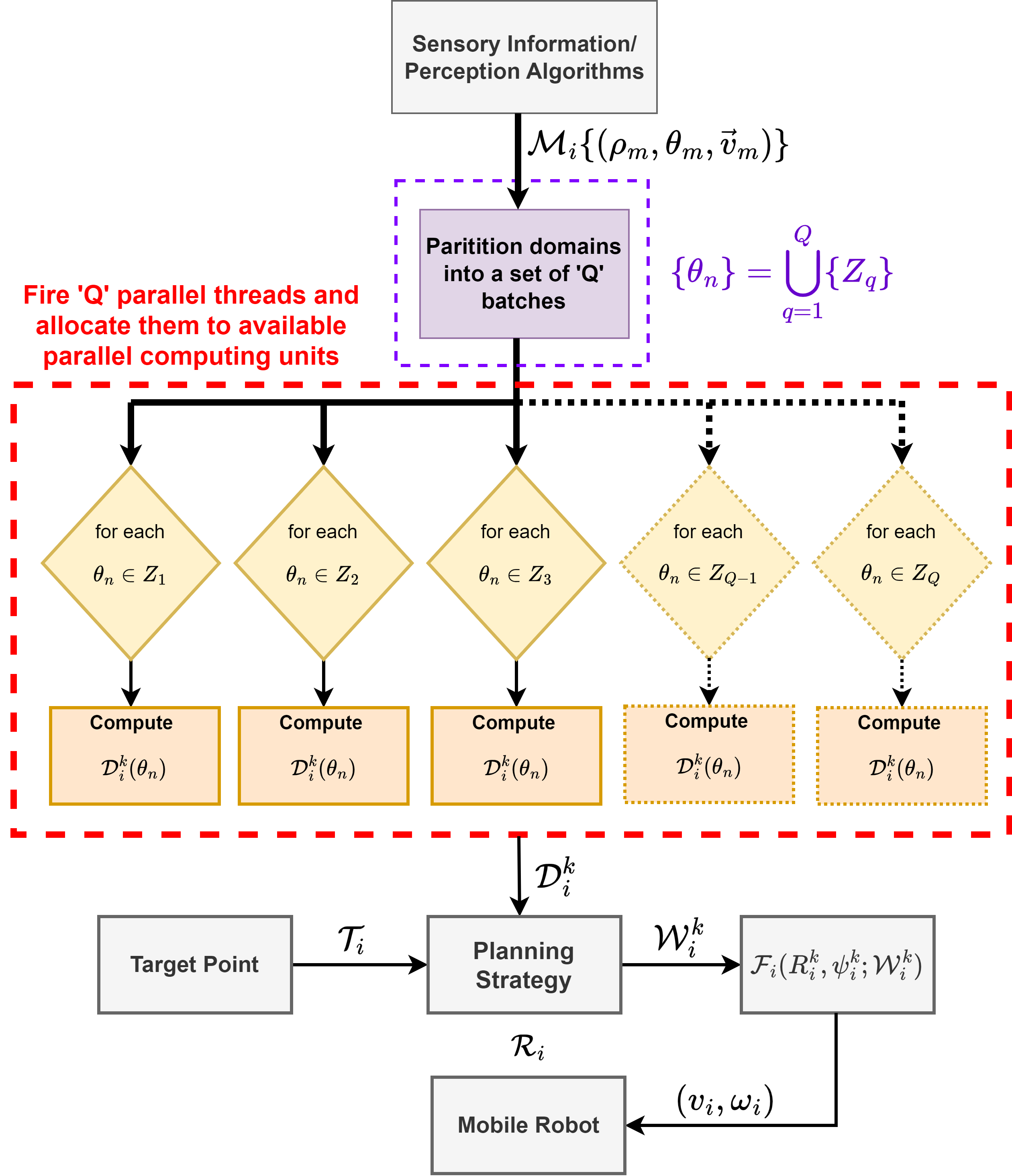}
    \caption{Parallelization in each iteration of Algorithm \ref{alg:Proposed algorithm}}
    \label{fig:ParallelFlow}
\end{figure}

\section{CONCLUSIONS \& FUTURE WORK}
\label{sec:6}
A novel, fast, and safe self-navigation algorithm for non-holonomic agents operating in crowded multi-agent scenarios is proposed in this work. The key novelties of this work are in avoiding any explicit collision avoidance method and in embedding the motion constraints of the unicycle in the proposed sensor-based controller design with a parallel implementation option. We develop an input-constrained feedback controller, and the development of a motion plan for guaranteed safe navigation by describing invariant sets as state constraints within the bounds of safety identified during run-time. The elegance of the proposed geometric structure of the invariant sets utilized in constructing the planner function is in deterministically dealing with uncertainties associated with the measurements/estimates obtained during run-time. 
Consequently, the proposed algorithm in this work is an independent sensor-based solution for each agent in any multi-agent navigation problem where it can have two components! coordinating and non-coordinating. The potential of the planner is in exploring learning-based methods for identifying intelligent strategies involving feedback controllers.

\bibliographystyle{unsrt}
\bibliography{reference.bib}

\begin{thebibliography}{10}

\bibitem{CooperativeCollisionAvoidanceNonHolRobots}
Javier Alonso-Mora, Paul Beardsley, and Roland Siegwart.
\newblock Cooperative collision avoidance for nonholonomic robots.
\newblock {\em IEEE Transactions on Robotics}, 34(2):404--420, 2018.

\bibitem{TrajOptMultiRobotNav}
Shravan Krishnan, Govind~Aadithya Rajagopalan, Sivanathan Kandhasamy, and Madhavan Shanmugavel.
\newblock Continuous-time trajectory optimization for decentralized multi-robot navigation.
\newblock {\em IFAC-PapersOnLine}, 53(1):494--499, 2020.
\newblock 6th Conference on Advances in Control and Optimization of Dynamical Systems ACODS 2020.

\bibitem{ExperimentalSafeMPC}
Ivo Batkovic, Ankit Gupta, Mario Zanon, and Paolo Falcone.
\newblock Experimental validation of safe mpc for autonomous driving in uncertain environments.
\newblock {\em IEEE Transactions on Control Systems Technology}, 31(5):2027--2042, 2023.

\bibitem{MultiTrajectoryMPCUAV}
Danilo Saccani, Leonardo Cecchin, and Lorenzo Fagiano.
\newblock Multitrajectory model predictive control for safe uav navigation in an unknown environment.
\newblock {\em IEEE Transactions on Control Systems Technology}, 31(5):1982--1997, 2023.

\bibitem{NMPCMultiRobot}
Amir {Salimi Lafmejani} and Spring Berman.
\newblock Nonlinear mpc for collision-free and deadlock-free navigation of multiple nonholonomic mobile robots.
\newblock {\em Robotics and Autonomous Systems}, 141:103774, 2021.

\bibitem{TubeMPCMultiRobot}
Xiang Chen and Steven Liu.
\newblock Robust decentralized multi robot navigation using tube based model predictive control and optimal reciprocal collision avoidance.
\newblock In {\em IECON 2022 – 48th Annual Conference of the IEEE Industrial Electronics Society}, pages 1--6, 2022.

\bibitem{CBFTheoryApplications}
Aaron~D. Ames, Samuel Coogan, Magnus Egerstedt, Gennaro Notomista, Koushil Sreenath, and Paulo Tabuada.
\newblock Control barrier functions: Theory and applications.
\newblock In {\em 2019 18th European Control Conference (ECC)}, pages 3420--3431, 2019.

\bibitem{CBFInputtoStateSafety}
Anil Alan, Andrew~J. Taylor, Chaozhe~R. He, Aaron~D. Ames, and Gábor Orosz.
\newblock Control barrier functions and input-to-state safety with application to automated vehicles.
\newblock {\em IEEE Transactions on Control Systems Technology}, 31(6):2744--2759, 2023.

\bibitem{BarrierFunctionsMultiRobotNavigation}
Paul Glotfelter, Jorge Cortés, and Magnus Egerstedt.
\newblock Nonsmooth barrier functions with applications to multi-robot systems.
\newblock {\em IEEE Control Systems Letters}, 1(2):310--315, 2017.

\bibitem{BarrierCertificatesMultiRobot}
Li~Wang, Aaron~D. Ames, and Magnus Egerstedt.
\newblock Safety barrier certificates for collisions-free multirobot systems.
\newblock {\em IEEE Transactions on Robotics}, 33(3):661--674, 2017.

\bibitem{ICCBF}
Devansh~R. Agrawal and Dimitra Panagou.
\newblock Safe control synthesis via input constrained control barrier functions.
\newblock In {\em 2021 60th IEEE Conference on Decision and Control (CDC)}, pages 6113--6118, 2021.

\bibitem{BREEDEN2023111109}
Joseph Breeden and Dimitra Panagou.
\newblock Robust control barrier functions under high relative degree and input constraints for satellite trajectories.
\newblock {\em Automatica}, 155:111109, 2023.

\bibitem{NMPCCBFCrowd}
Veronica Vulcano, Spyridon~G. Tarantos, Paolo Ferrari, and Giuseppe Oriolo.
\newblock Safe robot navigation in a crowd combining nmpc and control barrier functions.
\newblock In {\em 2022 IEEE 61st Conference on Decision and Control (CDC)}, pages 3321--3328, 2022.

\bibitem{NMPCApproxTraffic}
Makoto Obayashi and Gaku Takano.
\newblock Real-time autonomous car motion planning using nmpc with approximated problem considering traffic environment.
\newblock {\em IFAC-PapersOnLine}, 51(20):279--286, 2018.
\newblock 6th IFAC Conference on Nonlinear Model Predictive Control NMPC 2018.

\bibitem{CBFMPC}
Jun Zeng, Bike Zhang, and Koushil Sreenath.
\newblock Safety-critical model predictive control with discrete-time control barrier function.
\newblock In {\em 2021 American Control Conference (ACC)}, pages 3882--3889, 2021.

\bibitem{SociallyAwareNavigation}
Xuan-Tung Truong and Trung~Dung Ngo.
\newblock Toward socially aware robot navigation in dynamic and crowded environments: A proactive social motion model.
\newblock {\em IEEE Transactions on Automation Science and Engineering}, 14(4):1743--1760, 2017.

\bibitem{SOTApredictionNavperformance}
Sriyash Poddar, Christoforos Mavrogiannis, and Siddhartha~S. Srinivasa.
\newblock From crowd motion prediction to robot navigation in crowds, 2023.

\bibitem{CrowdAwareDRL}
Yulin Zhang and Zhengyong Feng.
\newblock Crowd-aware mobile robot navigation based on improved decentralized structured rnn via deep reinforcement learning.
\newblock {\em Sensors}, 23(4), 2023.

\bibitem{khalil2015nonlinear}
H.K. Khalil.
\newblock {\em Nonlinear Control, Global Edition}.
\newblock Pearson Education, 2015.

\bibitem{SetInvarianceControl}
F.~Blanchini.
\newblock Set invariance in control.
\newblock {\em Automatica}, 35(11):1747--1767, 1999.

\bibitem{Conner2011}
David~C. Conner, Howie Choset, and Alfred~A. Rizzi.
\newblock Integrating planning and control for single-bodied wheeled mobile robots.
\newblock {\em Autonomous Robots}, 30(3):243--264, Apr 2011.

\bibitem{SeqComp}
R.~R. Burridge, A.~A. Rizzi, and D.~E. Koditschek.
\newblock Sequential composition of dynamically dexterous robot behaviors.
\newblock {\em The International Journal of Robotics Research}, 18(6):534--555, 1999.

\bibitem{MultiRobotSeqComp}
Glenn Wagner, Howie Choset, and Avinash Siravuru.
\newblock Multirobot sequential composition.
\newblock In {\em 2016 IEEE/RSJ International Conference on Intelligent Robots and Systems (IROS)}, pages 2081--2088, 2016.

\bibitem{NonHolCartNav}
Vinutha Kallem, Adam~T. Komoroski, and Vijay Kumar.
\newblock Sequential composition for navigating a nonholonomic cart in the presence of obstacles.
\newblock {\em IEEE Transactions on Robotics}, 27(6):1152--1159, 2011.

\bibitem{IntSlidingMode}
Yuri Shtessel, Christopher Edwards, Leonid Fridman, and Arie Levant.
\newblock {\em Introduction: Intuitive Theory of Sliding Mode Control}, pages 1--42.
\newblock Springer New York, New York, NY, 2014.

\end{thebibliography}


\begin{thebibliography}{34}\vspace*{-12pt}

\bibitem{[1]} G. M. Amdahl, G. A. Blaauw, and F. P. Brooks, ``Architecture of the IBM System/360,'' {\em IBM J. Res. \& Dev.}, vol. 8, no. 2, pp. 87--101, 1964. (journal)
\end{thebibliography}

\begin{thebibliography}{34}
\item[]\hskip-18pt J. K. Author, ``Title of chapter in the book,'' in {\em Title of His Published Book}, xth ed. City of Publisher, (only U.S. State), Country: Abbrev. of Publisher, year, ch. x, sec. x, pp. xxx-xxx.
\end{thebibliography}

\begin{thebibliography}{34}\vspace*{-12pt}

\bibitem{}G. O. Young, ``Synthetic structure of industrial plastics,'' in {\em Plastics}, 2nd ed., vol. 3, J. Peters, Ed. New York, NY, USA: McGraw-Hill, 1964, pp. 15-64.

\bibitem{}W.-K. Chen, {\it Linear Networks and Systems}. Belmont, CA, USA: Wadsworth, 1993, pp. 123--135.

\end{thebibliography}

\vspace*{-12pt}

\begin{thebibliography}{34}
\item[]\hskip-18pt
J. K. Author, ``Name of paper,'' {\it Abbrev. Title of Periodical}, vol. {\it x},   no.~{\it x}, pp. xxx-xxx, Abbrev. Month, year, DOI. \href{https://dx.doi.org/10.1109.XXX.123456}{10.1109.XXX.123456}.
\end{thebibliography}

\vspace*{-12pt}

\begin{thebibliography}{34}
\setcounter{enumiv}{2}

\bibitem{}J. U. Duncombe, ``Infrared navigation---Part I: An assessment of feasibility,'' {\em IEEE Trans. Electron Devices}, vol. ED-11, no. 1, pp. 34--39,\break Jan. 1959,10.1109/TED.2016.2628402.

\bibitem{}E. P. Wigner, ``Theory of traveling-wave optical laser,''
{\em Phys. Rev.},\break  vol. 134, pp. A635--A646, Dec. 1965. 

\bibitem{}E. H. Miller, ``A note on reflector arrays,'' {\em IEEE Trans. Antennas Propagat.}, to be published.
\end{thebibliography}

\vspace*{-12pt}
\begin{thebibliography}{34}
\item[]\hskip-18pt
J. K. Author, ``Title of report,'' Abbrev. Name of Co., City of Co., Abbrev. State, Country, Rep. xxx, year.
\end{thebibliography}

\vspace*{-12pt}
\begin{thebibliography}{34}
\setcounter{enumiv}{5}

\bibitem{} E. E. Reber, R. L. Michell, and C. J. Carter, ``Oxygen absorption in the earth's atmosphere,'' Aerospace Corp., Los Angeles, CA, USA, Tech. Rep. TR-0200 (4230-46)-3, Nov. 1988.

\bibitem{} J. H. Davis and J. R. Cogdell, ``Calibration program for the 16-foot antenna,'' Elect. Eng. Res. Lab., Univ. Texas, Austin, TX, USA, Tech. Memo. NGL-006-69-3, Nov. 15, 1987.
\end{thebibliography}

\vspace*{-12pt}
\begin{thebibliography}{34}
\item[]\hskip-18pt
{\em Name of Manual/Handbook}, x ed., Abbrev. Name of Co., City of Co., Abbrev. State, Country, year, pp. xxx-xxx.
\end{thebibliography}

\vspace*{-12pt}

\begin{thebibliography}{34}
\setcounter{enumiv}{7}

\bibitem{} {\em Transmission Systems for Communications}, 3rd ed., Western Electric Co., Winston-Salem, NC, USA, 1985, pp. 44--60.

\bibitem{} {\em Motorola Semiconductor Data Manual}, Motorola Semiconductor Products Inc., Phoenix, AZ, USA, 1989.
\end{thebibliography}

\vspace*{-12pt}
\begin{thebibliography}{34}
\item[]\hskip-18pt
J. K. Author, ``Title of chapter in the book,'' in {\em Title of Published Book},\break xth ed. City of Publisher, State, Country: Abbrev. of Publisher, year, ch. x, sec. x, pp. xxx xxx. [Online]. Available: http://www.web.com 
\end{thebibliography}

\vspace*{-12pt}

\begin{thebibliography}{34}
\setcounter{enumiv}{9}

\bibitem{}G. O. Young, ``Synthetic structure of industrial plastics,'' in Plastics, vol. 3, Polymers of Hexadromicon, J. Peters, Ed., 2nd ed. New York, NY, USA: McGraw-Hill, 1964, pp. 15-64. [Online]. Available: http://www.bookref.com. 

\bibitem{} {\em The Founders' Constitution}, Philip B. Kurland and Ralph Lerner, eds., Chicago, IL, USA: Univ. Chicago Press, 1987. [Online]. Available: http://press-pubs.uchicago.edu/founders/

\bibitem{} The Terahertz Wave eBook. ZOmega Terahertz Corp., 2014. [Online]. Available: http://dl.z-thz.com/eBook/zomega\_ebook\_pdf\_1206\_sr.pdf. Accessed on: May 19, 2014. 

\bibitem{} Philip B. Kurland and Ralph Lerner, eds., {\em The Founders' Constitution}. Chicago, IL, USA: Univ. of Chicago Press, 1987, Accessed on: Feb. 28, 2010, [Online] Available: http://press-pubs.uchicago.edu/founders/ 
\end{thebibliography}

\vspace*{-12pt}
\begin{thebibliography}{34}
\item[]\hskip-18pt J. K. Author, ``Name of paper,'' {\em Abbrev. Title of Periodical}, vol. x, no. x,\break pp. xxx-xxx, Abbrev. Month, year. Accessed on: Month, Day, year, DOI: 10.1109.XXX.123456, [Online].
\end{thebibliography}

\vspace*{-12pt}

\begin{thebibliography}{34}
\setcounter{enumiv}{13}

\bibitem{}J. S. Turner, ``New directions in communications,'' {\em IEEE J. Sel. Areas Commun.}, vol. 13, no. 1, pp. 11-23, Jan. 1995. 

\bibitem{} W. P. Risk, G. S. Kino, and H. J. Shaw, ``Fiber-optic frequency shifter using a surface acoustic wave incident at an oblique angle,'' {\em Opt. Lett.}, vol. 11, no. 2, pp. 115-117, Feb. 1986. 

\bibitem{} P. Kopyt {\em et al.}, ``Electric properties of graphene-based conductive layers from DC up to terahertz range,'' {\em IEEE THz Sci. Technol.}, to be published. DOI: \href{https://dx.doi.org/10.1109.XXX.123456}{10.1109/TTHZ.2016.2544142}.
\end{thebibliography}

\vspace*{-12pt}
\begin{thebibliography}{34}
\item[]\hskip-18pt J.K. Author. (year, month). Title. presented at abbrev. conference title. [Type of Medium]. Available: site/path/file
\end{thebibliography}

\vspace*{-12pt}

\begin{thebibliography}{34}
\setcounter{enumiv}{16}

\bibitem{}PROCESS Corporation, Boston, MA, USA. Intranets: Internet technologies deployed behind the firewall for corporate productivity. Presented at INET96 Annual Meeting. [Online]. Available: http://home.process.com/Intranets/wp2.htp
\end{thebibliography}

\vspace*{-12pt}
\begin{thebibliography}{34}
\item[]\hskip-18pt J. K. Author. ``Title of report,'' Company. City, State, Country. Rep. no., (optional: vol./issue), Date. [Online] Available: site/path/file 
\end{thebibliography}

\vspace*{-12pt}

\begin{thebibliography}{34}
\setcounter{enumiv}{17}

\bibitem{}R. J. Hijmans and J. van Etten, ``Raster: Geographic analysis and modeling with raster data,'' R Package Version 2.0-12, Jan. 12, 2012. [Online]. Available: http://CRAN.R-project.org/package=raster 

\bibitem{}Teralyzer. Lytera UG, Kirchhain, Germany [Online]. Available: http://www.lytera.de/Terahertz\_THz\_Spectroscopy.php?id=home, Accessed on: Jun. 5, 2014.
\end{thebibliography}

\vspace*{-12pt}
\begin{thebibliography}{34}
\item[]\hskip-18pt Legislative body. Number of Congress, Session. (year, month day). {\em Number of bill or resolution, Title}. [Type of medium]. Available: site/path/file

\item[]\hskip-18pt {\em NOTE:} ISO recommends that capitalization follow the accepted practice for the language or script in which the information is given.
\end{thebibliography}

\vspace*{-12pt}

\begin{thebibliography}{34}
\setcounter{enumiv}{19}

\bibitem{}U. S. House. 102nd Congress, 1st Session. (1991, Jan. 11). {\em H. Con. Res. 1, Sense of the Congress on Approval of Military Action}. [Online]. Available: LEXIS Library: GENFED File: BILLS. 
\end{thebibliography}

\vspace*{-12pt}
\begin{thebibliography}{34}
\item[]\hskip-18pt Name of the invention, by inventor's name. (year, month day). Patent Number [Type of medium]. Available:site/path/file
\end{thebibliography}

\vspace*{-12pt}

\begin{thebibliography}{34}
\setcounter{enumiv}{20}

\bibitem{}Musical tooth brush with mirror, by L.M.R. Brooks. (1992, May 19). Patent D 326 189
[Online]. Available: NEXIS Library: LEXPAT File:   DES. 
\end{thebibliography}

\vspace*{-12pt}
\begin{thebibliography}{34}
\item[]\hskip-18pt J. K. Author, ``Title of paper,'' in {\em Abbreviated Name of Conf.}, City of Conf., Abbrev. State (if given), Country, year, pp. xxx-\break xxx.
\end{thebibliography}

\vspace*{-12pt}

\begin{thebibliography}{34}
\setcounter{enumiv}{21}

\bibitem{}D. B. Payne and J. R. Stern, ``Wavelength-switched passively coupled single-mode optical network,'' in {\em Proc. IOOC-ECOC}, Boston, MA, USA, 1985,
pp. 585-590. 

\end{thebibliography}

\vspace*{-12pt}

\begin{thebibliography}{34}
\setcounter{enumiv}{22}

\bibitem{}D. Ebehard and E. Voges, ``Digital single sideband detection for inter ferometric sensors,'' presented at the {\em 2nd Int. Conf. Optical Fiber Sensors}, Stuttgart, Germany, Jan. 2-5, 1984.
\end{thebibliography}

\vspace*{-12pt}
\begin{thebibliography}{34}
\item[]\hskip-18pt J. K. Author, ``Title of patent,'' U.S. Patent x xxx xxx, Abbrev. Month, day, year.
\end{thebibliography}

\vspace*{-12pt}

\begin{thebibliography}{34}
\setcounter{enumiv}{23}

\bibitem{}G. Brandli and M. Dick, ``Alternating current fed power supply,'' U.S. Patent 4 084 217, Nov. 4, 1978.
\end{thebibliography}

\vspace*{-12pt}
\begin{thebibliography}{34}\itemindent-6.5pt
\item[a)] J. K. Author, ``Title of thesis,'' M.S. thesis, Abbrev. Dept., Abbrev. Univ., City of Univ., Abbrev. State, year.

\item[b)] J. K. Author, ``Title of dissertation,'' Ph.D. dissertation, Abbrev. Dept., Abbrev. Univ., City of Univ., Abbrev. State, year.
\end{thebibliography}

\vspace*{-12pt}

\begin{thebibliography}{34}
\setcounter{enumiv}{24}

\bibitem{}J. O. Williams, ``Narrow-band analyzer,'' Ph.D. dissertation, Dept. Elect. Eng., Harvard Univ., Cambridge, MA, USA, 1993.

\bibitem{}N. Kawasaki, ``Parametric study of thermal and chemical nonequilibrium nozzle flow,'' M.S. thesis, Dept. Electron. Eng., Osaka Univ., Osaka, Japan, 1993.
\end{thebibliography}

\vspace*{-12pt}
\begin{thebibliography}{34}\itemindent-6.5pt
\item[a)] J. K. Author, private communication, Abbrev. Month, year.

\item[b)] J. K. Author, ``Title of paper,'' unpublished.

\item[c)] J. K. Author, ``Title of paper,'' to be published.
\end{thebibliography}

\vspace*{-12pt}

\begin{thebibliography}{34}
\setcounter{enumiv}{26}

\bibitem{}A. Harrison, private communication, May 1995.

\bibitem{}B. Smith, ``An approach to graphs of linear forms,'' unpublished.

\bibitem{}A. Brahms, ``Representation error for real numbers in binary computer arithmetic,'' IEEE Computer Group Repository, Paper R-67-85.
\end{thebibliography}

\vspace*{-12pt}
\begin{thebibliography}{34}\itemindent-6.5pt
\item[a)] {Title of Standard}, Standard number, date.

\item[b)] {Title of Standard}, Standard number, Corporate author, location, date.
\end{thebibliography}

\vspace*{-12pt}

\begin{thebibliography}{34}
\setcounter{enumiv}{29}
\bibitem{}IEEE Criteria for Class IE Electric Systems, IEEE Standard 308, 1969.

\bibitem{} Letter Symbols for Quantities, ANSI Standard Y10.5-1968.
\end{thebibliography}

\vspace*{-12pt}

\begin{thebibliography}{34}

\setcounter{enumiv}{31}

\bibitem{}R. Fardel, M. Nagel, F. Nuesch, T. Lippert, and A. Wokaun, ``Fabrication of organic light emitting diode pixels by laser-assisted forward transfer,'' {\em Appl. Phys. Lett.}, vol. 91, no. 6, Aug. 2007, Art. no. 061103. 

\bibitem{} J. Zhang and N. Tansu, ``Optical gain and laser characteristics of InGaN quantum wells on ternary InGaN substrates,'' {\em IEEE Photon.} J., vol. 5, no. 2, Apr. 2013, Art. no. 2600111. 
\end{thebibliography}

\vspace*{-12pt}

\begin{thebibliography}{34}
\setcounter{enumiv}{33}

\bibitem{}S. Azodolmolky {\em et al.}, Experimental demonstration of an impairment aware network planning and operation tool for transparent/translucent optical networks,'' {\em J. Lightw. Technol.}, vol. 29, no. 4, pp. 439-448, Sep. 2011.
\end{thebibliography}

\begin{IEEEbiography}[{\includegraphics[width=2.5cm,height=3.2cm,keepaspectratio]{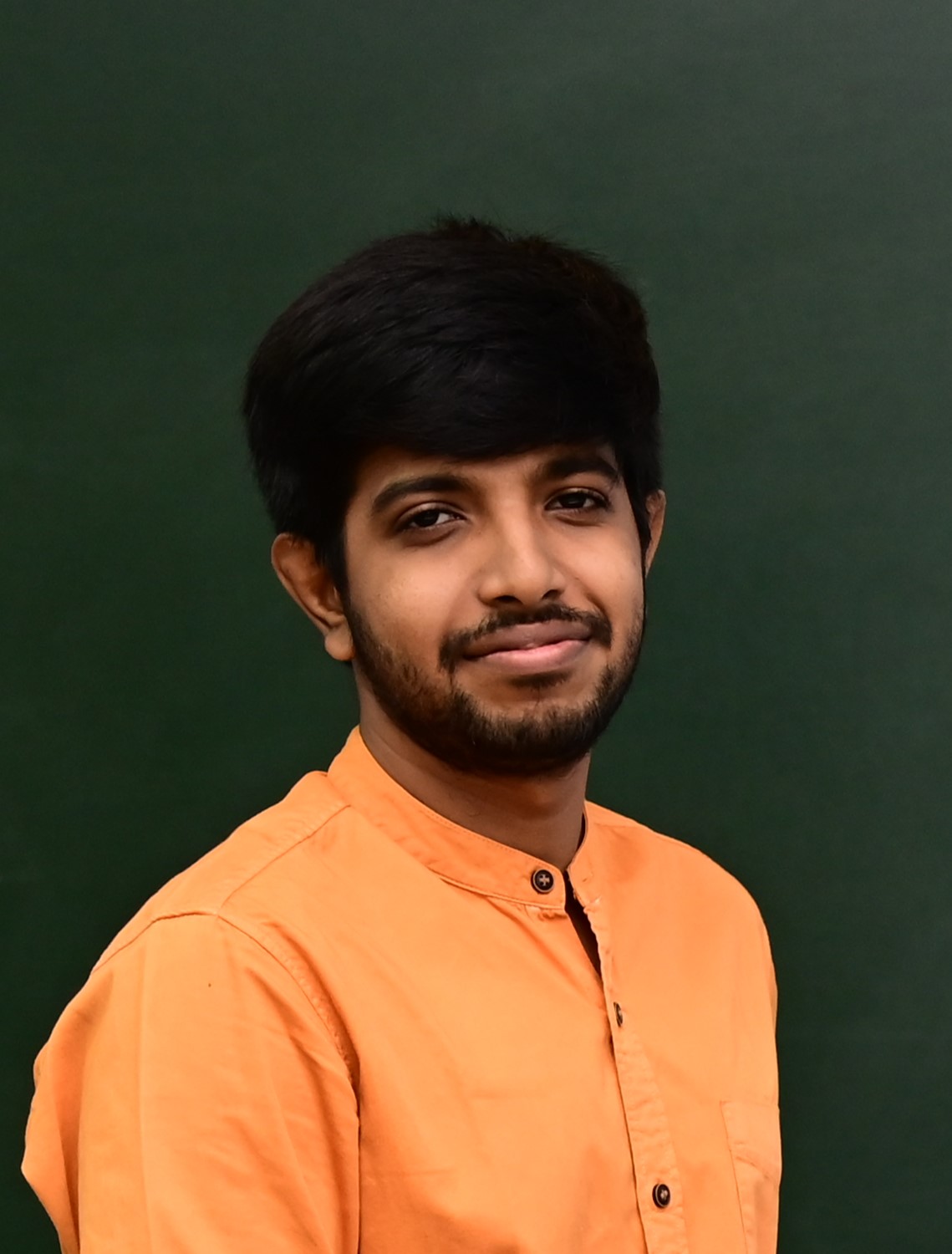}}]{VEEJAY KARTHIK J}{\space} received the B.Tech degree in Electrical and Electronics Engineering from National Institute of Technology Tiruchirappalli, Tamilnadu, India. He is currently pursuing M.Tech+PhD Dual Degree in Systems and Control Engineering at the Indian Institute of Technology Bombay, Mumbai, India. He is a recipient of the prestigious Prime Minister’s Research Fellowship. His primary research interests primarily lie in solving planning and control problems in mobile robotics from a holistic point of view.
\end{IEEEbiography}

\begin{IEEEbiography}[{\includegraphics[width=2.5cm,height=3.2cm,keepaspectratio]{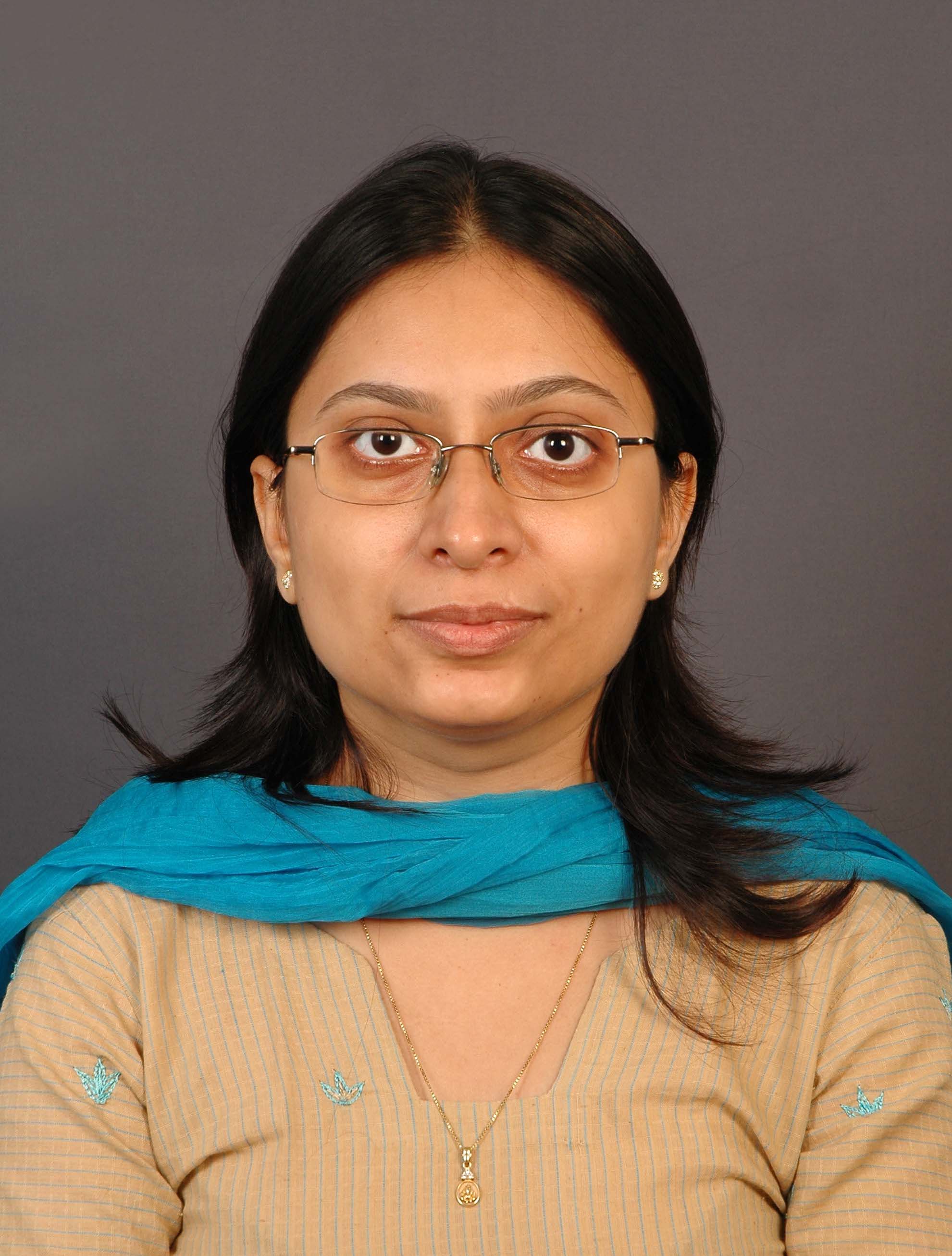}}]{LEENA VACHHANI}{\space} received the Ph.D. degree in embedded robotics from IIT Madras, Chennai, India, in 2009. She is currently a Professor with the Systems and Control Engineering Group, Indian Institute of Technology Bombay, Mumbai, India. She has contributed in the areas of embedded control and robotic applications that include topics on multiagent mapping, exploration, patrolling, and coverage. She has developed laboratories on embedded control systems, autonomous robots, and multiagent systems with unique concepts. She has been a Faculty Advisor of AUV (autonomous underwater vehicle)-IITB team since its inception in 2010. She is also Professor-in-Charge of Technology Innovation Hub (TIH), IIT Bombay, established under National Mission on Interdisciplinary Cyber-Physical Systems. Her current research interests include perception modeling for single- and multiagent applications, edge computing for IoT, and multiagent applications for IoT framework.
\end{IEEEbiography}

\end{document}